\documentclass[notitlepage]{report}

\usepackage{arxiv}

\usepackage[ruled,vlined,linesnumbered]{algorithm2e}
\usepackage{amsfonts}       % blackboard math symbols
\usepackage{amssymb}
\usepackage{amsthm}
\usepackage{booktabs}       % professional-quality tables
\usepackage{bm}
\usepackage{enumitem}
\usepackage[T1]{fontenc}    % use 8-bit T1 fonts
\usepackage{hyperref}       % hyperlinks
\usepackage[utf8]{inputenc} % allow utf-8 input
\usepackage{makecell}
\usepackage{mathtools}
\usepackage{microtype}      % microtypography
\usepackage{multirow}
\usepackage{nicefrac}       % compact symbols for 1/2, etc.
\usepackage{subfig}
\usepackage{thmtools}
\usepackage{url}            % simple URL typesetting
\usepackage{xcolor}         % colors

\hypersetup{
  colorlinks=true,
  linkcolor=blue,
  citecolor=blue
}

\SetKwInput{KwInput}{Input}
\SetKwInput{KwInit}{Initialize}
\SetKwInput{KwTest}{Test}
\SetKwInput{KwDef}{Definition}
\SetKwInput{KwReturn}{Return}

\newcommand{\bftab}{\fontseries{b}\selectfont}

\definecolor{cost}{HTML}{B03A2E}
\definecolor{reward}{HTML}{000000}
\newcommand{\cost}[1]{\textcolor{cost}{#1}}
\newcommand{\reward}[1]{\textcolor{reward}{#1}}

\newtheorem{theorem}{Theorem}
\newtheorem{lemma}{Lemma}

\newtheorem{definition}{Definition}

\newtheorem{remark}{Remark}
\newtheorem{assumption}{Assumption}

\newtheorem*{assumption*}{Assumption}

\declaretheoremstyle[
notefont=\normalfont\itshape,
notebraces={}{},
headfont=\normalfont\itshape,
bodyfont=\normalfont,
headformat=\NAME\NOTE
]{nopar}

\declaretheorem[name={Proof of}, style=nopar]{refproof}

\usepackage[
  backend=biber,
  style=numeric-comp,
  maxbibnames=9,
  maxcitenames=1,
  uniquelist=false,
  uniquename=false,
  isbn=false,
  sorting=nty,
  doi=false
]{biblatex}
\DeclareMathOperator*{\argmax}{argmax}
\DeclareMathOperator*{\argmin}{argmin}

\renewcommand{\thesection}{\arabic{section}}

\newcount\Comments  % 0 suppresses notes to selves in text
\newcommand{\kibitz}[2]{\ifnum\Comments=1\textcolor{#1}{#2}\fi}

\newcounter{parentnumber}

\DeclareSourcemap{
  \maps[datatype=bibtex, overwrite]{
    \map{
      \step[fieldset=address, null]
      \step[fieldset=publisher, null]
      \step[fieldset=url, null]
      \step[fieldset=urldate, null]
      \step[fieldset=isbn, null]
      \step[fieldset=issn, null]
      \step[fieldset=number, null]
      \step[fieldset=doi, null]
      \step[fieldset=abstract, null]
      \step[fieldset=volume, null]
      \step[fieldset=pages, null]
      \step[fieldset=language, null]
      \step[fieldset=month, null]
      \step[fieldset=series, null]
      \step[fieldset=file, null]
      \step[fieldset=note, null]
    }
  }
}

\title{A Primal-Dual-Critic Algorithm for\\Offline Constrained Reinforcement Learning}
\author{
  Kihyuk Hong \\
  University of Michigan \\
  \texttt{kihyukh@umich.edu} \\
  \And
  Yuhang Li \\
  University of Michigan \\
  \texttt{liyuhang@umich.edu} \\
  \And
  Ambuj Tewari \\
  University of Michigan \\
  \texttt{tewaria@umich.edu} \\
}
\date{}

\begin{document}

\maketitle

\begin{abstract}
Offline constrained reinforcement learning (RL) aims to learn a policy that maximizes the expected cumulative reward subject to constraints on expected cumulative cost using an existing dataset.
In this paper, we propose Primal-Dual-Critic Algorithm (PDCA), a novel algorithm for offline constrained RL with general function approximation.
PDCA runs a primal-dual algorithm on the Lagrangian function estimated by critics.
The primal player employs a no-regret policy optimization oracle to maximize the Lagrangian estimate and the dual player acts greedily to minimize the Lagrangian estimate.
We show that PDCA can successfully find a near saddle point of the Lagrangian, which is nearly optimal for the constrained RL problem.
Unlike previous work that requires concentrability and a strong Bellman completeness assumption, PDCA only requires concentrability and realizability assumptions for sample-efficient learning.
\end{abstract}

\section{INTRODUCTION}

Offline constrained reinforcement learning (RL) aims to learn a decision making policy that performs well while satisfying safety constraints given a dataset of trajectories collected from historical experiments.
It enjoys the benefits of offline RL \parencite{levine_offline_2020}:
not requiring interaction with the environment enables real-world applications where collecting interaction data is expensive (e.g., robotics \parencite{kumar_workflow_2021,levine_learning_2018}) or dangerous (e.g., healthcare \parencite{tang_model_2021}).
It also enjoys the benefits of constrained RL \parencite{altman_constrained_1999}: being able to specify constraints to the behavior of the agent enables real-world applications with safety concerns (e.g., smart grid \parencite{wang_safe_2020}, robotics \parencite{gu_deep_2017}).

Offline constrained RL with function approximation (e.g., neural networks) is of particular interest because function approximation can encode inductive biases to allow sample-efficient learning in large state spaces.
As is the case for offline unconstrained RL \parencite{ozdaglar2023revisiting,xie2021bellman}, offline constrained RL with function approximation requires two classes of assumptions for sample-efficient learning.

The first class of assumptions, called \emph{representational assumptions}, requires the learner to have access to a sufficiently rich value function class that models action value functions of policies.
The mildest representational assumption is the realizability assumption that requires the action value functions of candidate policies to be captured by the function class.
A stronger representational assumption is the Bellman completeness assumption that requires the function class to be closed under the Bellman operator.

The second class of assumptions, called \emph{data coverage assumptions}, requires the offline dataset to be rich enough to cover the state-action distributions induced by target policies.
The assumptions address a major challenge in offline RL called \emph{distribution} shift, which refers to the mismatch of the state-action distributions induced by candidate policies and the distribution in the offline dataset.
The most commonly used data coverage assumption is concentrability \parencite{munos_error_2003,munos_error_2005}, which limits the norm of the ratio of state-action distribution induced by candidate policies to that induced by the behavior policy that generated the offline dataset.
%\textcite{chen_information-theoretic_2019} show that even with a strong representation assumption, low concentrability is required for sample-efficient learning.

Previous works on offline RL with function approximation require either a strong assumption on data coverage \parencite{xie_batch_2021} (stronger than concentrability) or a strong representational assumption \parencite{munos_finite-time_2008,antos_learning_2008,xie2021bellman,cheng_adversarially_2022} (stronger than realizability).
\textcite{chen_information-theoretic_2019} conjectured that concentrability and realizability of value functions are not sufficient for sample-efficient offline RL.
\textcite{foster_offline_2022} confirmed this by providing an information-theoretic lower bound which shows that concentrability and value function realizability are not sufficient for sample efficient offline RL.

Recently, a line of research on offline unconstrained RL emerged that only requires concentrability and realizability assumptions \parencite{xie_q_2020,zhan_offline_2022,zhu_importance_2023}.
In particular, they do not require Bellman completeness assumption, which is a strong representational assumption~\parencite{zhan_offline_2022,zanette_bellman_2022}.
They do not contradict the impossibility result by \textcite{foster_offline_2022} because they make an additional realizability assumption on the \textit{marginalized importance weights} (MIW; ratio of state-action distribution induced by policy to data distribution).
Motivated by their work, we propose a sample-efficient algorithm for offline constrained RL with function approximation that requires concentrability, value function realizability and MIW realizability assumptions.
We make the following contributions.
\begin{itemize}[leftmargin=*]
\item We show a sample complexity bound that scales with a concentrability measure, $1/\epsilon^2$ and a dimensionality measure of function classes, for finding a nearly optimal policy with suboptimality $\epsilon$ that approximately satisfies the cost constraints under the assumptions of value function realizability, concentrability, and MIW realizability of an optimal policy.
We do not require Bellman completeness, a strong representational assumption, required by previous work.
\item Our algorithm takes as an input a target cost threshold. By using a target cost threshold stricter than the desired threshold, the algorithm can produce a nearly optimal policy that \textit{exactly} satisfies the desired constraints with the same sample complexity.
\item We study the case where the function class for MIW is misspecified and does not realize the MIW of an optimal policy.
In this case, our algorithm can still find a policy at least as good as any policy of which MIW is realized by the function class but the sample complexity bound is suboptimal and scales with $1 / \epsilon^4$.
\item Benchmark experiments show that the empirical performance of our algorithm generally matches or outperforms the state-of-the-art practical algorithms COptiDICE and CPQ that produce Markovian policies.
\end{itemize}

\subsection{Related Work}

\paragraph{Offline RL without Completeness Assumption}

There is a recent line of works on offline \emph{unconstrained} RL that removes the Bellman completeness assumption by assuming MIW realizability.
\textcite{xie_q_2020} propose a Q-value based algorithm called MABO that learns the optimal Q-value function by solving a minimax optimization problem.
They require all-policy realizability of value functions, all-policy concentrability and all-policy marginalized importance weight realizability.
\textcite{zhan_offline_2022} propose a linear programming based algorithm called PRO-RL that regularizes the objective function to discourage distribution shift.
They only require single-policy realizability of both value function and marginalized importance weight, and only require single-policy concentrability.
However, their sample complexity is suboptimal ($\sim 1 / \epsilon^6$).
\textcite{zhu_importance_2023} propose an actor-critic based algorithm called A-Crab.
They require all-policy value function realizability, single-policy concentrability and single-policy marginalized importance weight realizability.

\paragraph{Offline Constrained RL}

The only work on provably sample efficient offline constrained RL with function approximation, to the best of our knowledge, is by \textcite{le_batch_2019} who propose a provably sample-efficient primal-dual algorithm that uses the fitted-Q iteration algorithm as a subroutine for updating the primal variable and a no-regret online algorithm for updating the dual variable.
Their analysis requires all-policy concentrability and Bellman completeness assumptions.
Our work improves over \textcite{le_batch_2019} by weakening the Bellman completeness assumption.

\paragraph{Practical Algorithms for Offline Constrained RL}

There are recent works on practical algorithms for offline constrained RL \emph{without} provable guarantees.
\textcite{lee_coptidice_2022} propose an algorithm called COptiDICE, which is motivated by the linear programming approach for solving RL.
\textcite{liu2023constrained} propose CDT, an adaptation of the decision transformer framework for offline RL \parencite{chen2021decision} to the offline constrained RL setting.
\textcite{xu2022constraints} propose CPQ, a Q-learning based algorithm that penalizes out of distribution actions.
\textcite{liu2023datasets} provide datasets and benchmarks of aforementioned algorithms.

We compare our theoretical guarantees with previous works in Table~\ref{table:comparison}.
The first three rows are works on offline \emph{unconstrained} RL with function approximation that do not assume Bellman completeness.
The remaining rows are works on offline \emph{constrained} RL with function approximation.
The column $N$ shows how the sample complexity bound scales with the error tolerance $\epsilon$.
$Q^\pi$ is the value function for the policy $\pi$; $w^\pi$ is the marginalized importance weight of the policy $\pi$; $\text{sp}$ is the span function; $\pi^\star$ is the optimal policy.
$\mathcal{T}^\pi$ is the Bellman operator and $\forall \pi, \,\mathcal{T}^\pi f \in \mathcal{F}$ means Bellman completeness.
The notations used in the table are formally defined in the next section.

\begin{table*}[t]
\caption{Comparison of algorithms for offline (constrained) RL with function approximation}
\label{table:comparison}
\centering
\begin{tabular}{cccccc}
 \toprule
 \multirow{2}{*}{Algorithm} & \multirow{2}{*}{\makecell{Supports \\ constraints}} & \multicolumn{3}{c}{Assumptions} & \multirow{2}{*}{$N$} \\
  \cline{3-5}
  &  & Representation & Data coverage & MIW & \\
 \midrule
 MABO \parencite{xie_q_2020} & No & $\forall \pi, Q^\pi \in \mathcal{F}$ & $\forall \pi, \Vert w^\pi \Vert_{2, \mu} \leq C_{\ell_2}$ & $\forall \pi, w^\pi \in \text{sp}(\mathcal{W})$ &  $1 / \epsilon^2$ \\
 PRO-RL \parencite{zhan_offline_2022} & No & $Q^{\pi^\star} \in \mathcal{F}$ & $\Vert w^{\pi^\star} \Vert_\infty \leq C_\infty^\star$ & $w^{\pi^\star} \in \mathcal{W}$ &  $1 / \epsilon^6$ \\
 A-Crab \parencite{zhu_importance_2023} & No & $\forall \pi, Q^\pi \in \mathcal{F}$ & $\Vert w^{\pi^\star} \Vert_{2, \mu} \leq C_{\ell_2}^\star$ & $w^{\pi^\star} \in \mathcal{W}$ &  $1 / \epsilon^2$ \\
 \midrule
 MBCL \parencite{le_batch_2019} & Yes & $\forall \pi, f; \mathcal{T}^\pi f \in \mathcal{F}$ & $\forall \pi, \Vert w^\pi \Vert_\infty \leq C_\infty$ &  &  $1 / \epsilon^2$ \\
 PDCA (Ours) & Yes & $\forall \pi, Q^\pi \in \mathcal{F}$ & $\forall \pi, \Vert w^\pi \Vert_{2, \mu} \leq C_{\ell_2}$ & $w^{\pi^\star} \in \mathcal{W}$ &  $1 / \epsilon^2$ \\
 \bottomrule
\end{tabular}
\end{table*}

Compared to the work by \textcite{le_batch_2019}, we relax the Bellman completeness assumption at the expense of introducing a MIW realizability of an optimal policy.
The MIW realizability is a mild assumption since function class $\mathcal{W}$ only needs to include the MIW of an \emph{optimal} policy.
Moreover, we show in Theorem~\ref{thm:main-2} that even when $\mathcal{W}$ does not realize the MIW of an optimal policy, our algorithm can find a policy that is at least as good as any policy whose MIW is realizable by $\mathcal{W}$.
This result allows robustness against misspecification of $\mathcal{W}$.

\section{PRELIMINARIES \& NOTATIONS}

\paragraph{Notation}

We denote by $\Delta(\mathcal{X})$ the probability simplex over a finite set $\mathcal{X}$.
We denote by $\mathbb{R}_+$ the set of nonnegative real numbers.
We write $\bm\Delta^I = \{ \bm{x} \in \mathbb{R}_+^I : \sum_{i = 1}^I x_i \leq 1 \}$.
We denote by $\text{Unif}(\mathcal{X})$ the uniform distribution over $\mathcal{X}$.
We write $[N] = \{1, \dots, N\}$ for a natural number $N$.
We write $\bm{1} = (1, \dots, 1)$ and $\bm{0} = (0, \dots, 0)$.
We write $(\cdot)_+ = \max \{0, \cdot\}$.

\subsection{Constrained Markov Decision Process}

We formulate offline constrained RL using an infinite-horizon discounted constrained Markov decision process (CMDP)\parencite{altman_constrained_1999} defined by a tuple $\mathcal{M} = \left( \mathcal{S}, \mathcal{A}, P, R, \{C_i\}_{i = 1}^I, \gamma, s_0 \right)$,
where $\mathcal{S}$ is the state space,
$\mathcal{A}$ is the action space,
$P : \mathcal{S} \times \mathcal{A} \rightarrow \Delta(\mathcal{S})$ is the transition probability kernel,
$R : \mathcal{S} \times \mathcal{A} \rightarrow [0, 1]$ is the reward function,
$C_i : \mathcal{S} \times \mathcal{A} \rightarrow [0, 1], i = 1, \dots, I$ are the cost functions,
$\gamma \in (0, 1)$ is the discount factor,
and $s_0 \in \mathcal{S}$ is the initial state.
We assume $R$ and $C_i$, $i = 1, \dots, I$ are known to the learner while $P$ is unknown.

A stationary policy $\pi : \mathcal{S} \rightarrow \Delta(\mathcal{A})$ maps each state to a probability distribution over the action space.
Each policy $\pi$, together with the transition probability kernel $P$, induces a discounted occupancy measure $d^\pi : \mathcal{S} \times \mathcal{A} \rightarrow [0, 1]$ defined as $d^\pi(s, a) \coloneqq (1 - \gamma) \sum_{t = 0}^\infty \gamma^t P^\pi(s_t = s, a_t = a)$ where $P^\pi$ is the probability measure on the trajectory $(s_0, a_0, s_1, a_1, \dots)$ induced by the interaction of $\pi$ and $P$.
The value of a policy $\pi$ for a function $U : \mathcal{S} \times \mathcal{A} \rightarrow [0, 1]$ is the expected discounted cumulative values when executing $\pi$.
It is denoted by $J_U(\pi) \coloneqq \mathbb{E}^\pi \left[ \sum_{t = 0}^\infty \gamma^t U(s_t, a_t) \right]$ where $\mathbb{E}^\pi$ is the expectation over the randomness of the trajectory $(s_0, a_0, s_1, a_1, \dots)$ induced by $\pi$ and $P$.
Note that $J_U(\pi) \in [0, \frac{1}{1 - \gamma}]$ and $(1 - \gamma) J_U(\pi) = \mathbb{E}_\pi[U(s, a)]$ where we use the shorthand $\mathbb{E}_\pi[ \cdot ]$ for $\mathbb{E}_{(s, a) \sim d^\pi}[ \cdot ]$.
The Q-value function of a policy $\pi$ for a function $U: \mathcal{S} \times \mathcal{A} \rightarrow [0, 1]$ is denoted by $Q_U^\pi(s, a) \coloneqq \mathbb{E}^\pi \left[ \sum_{t = 0}^\infty \gamma^t U(s_t, a_t) ~|~ s_0 = s, a_0 = a \right]$.

\subsection{Function Approximation}

We assume access to a policy class $\Pi \subseteq ( \pi: \mathcal{S} \rightarrow \Delta(\mathcal{A}))$ consisting of candidate policies.
We assume access to a function class $\mathcal{F} \subseteq (\mathcal{S} \times \mathcal{A} \rightarrow [0, \frac{1}{1 - \gamma}] )$ that models the Q-value functions for the reward $R$ and the costs $C_1, \dots, C_I$.
We make the following realizability assumption on $\mathcal{F}$.

\begin{assumption}[Value function realizability] \label{assumption:value-realizability}
For any policy $\pi \in \Pi$, we have $Q^\pi_R \in \mathcal{F}$ and $Q^\pi_{C_i} \in \mathcal{F}$ for all $i = 1, \dots, I$.
\end{assumption}

Unlike \textcite{le_batch_2019}, we do not assume Bellman completeness that requires $\mathcal{T}_U^\pi f \in \mathcal{F}$ for all $\pi \in \Pi$ and $f \in \mathcal{F}$ where $\mathcal{T}_U^\pi$ is the Bellman operator defined by $(\mathcal{T}_U^\pi f)(s, a) = U(s, a) + \gamma \mathbb{E}_{s' \sim P(\cdot|s, a)}[f(s', \pi)]$.
As \textcite{zhan_offline_2022} and \textcite{zanette_when_2022} discuss, it is a strong assumption hard to meet and has an unnatural non-monotone property: adding a function to the function class may make the function class violate Bellman completeness.

\subsection{Offline Constrained RL}

Offline constrained RL aims to find a policy $\pi : \mathcal{S} \rightarrow \Delta(\mathcal{A})$ among a given policy class $\Pi$ that maximizes $J_R(\pi)$ while satisfying the constraints $J_{C_i}(\pi) \leq \tau_i$ for all $i = 1, \dots, I$, where the thresholds $\tau_i \in [0, \frac{1}{1 - \gamma}]$ are given.
Offline constrained RL can be written as an optimization problem $\mathcal{P}(\bm\tau)$ defined as follows.

\begin{definition}[Optimization problem] \label{def:opt}
Given cost thresholds $\bm\tau = (\tau_1, \dots, \tau_I)$, we denote by $\mathcal{P}(\bm\tau)$ the following optimization problem.
\begin{equation} \label{eqn:opt} \tag{OPT}
\begin{aligned}
\max_{\pi \in \mathrm{Conv}(\Pi)}&~~J_R(\pi) \\
\mathrm{subject\ to}&~~J_{C_i}(\pi) \leq \tau_i, ~~i = 1, \dots I.
\end{aligned}
\end{equation}
\end{definition}

As done in \textcite{le_batch_2019}, instead of optimizing over the policy class $\Pi$, we optimize over its convex hull denoted by $\text{Conv}(\Pi)$.
The convex hull $\text{Conv}(\Pi)$ contains all policy mixtures of the form $\sum_{j = 1}^m \beta_j \pi_j$ where $\pi_1, \dots, \pi_m \in \Pi$, $\beta_j \geq 0$ for $j = 1, \dots, m$ and $\sum_{j = 1}^m \beta_j = 1$.
A policy mixture $\sum_{j = 1}^m \beta_j \pi_j$ is executed by sampling a single policy from $\pi_1, \dots, \pi_m$ according to the distribution $(\beta_1, \dots, \beta_m)$, and then executing the sampled policy for the entire trajectory.
Viewing the problem in the occupancy measure space, (\ref{eqn:opt}) can be seen as
\begin{equation} \label{eqn:opt-occupancy}
\begin{aligned}
\max_{\nu \in \text{Conv}(\mathcal{V})} &~~\langle \nu, R \rangle \\
\text{subject to} &~~\langle \nu, C_i \rangle \leq \tau_i, ~~i = 1, \dots, I
\end{aligned}
\end{equation}
where $\mathcal{V} = \{ d^\pi : \pi \in \Pi \}$ is the set of occupancy measures of policies in $\Pi$ and $d^\pi$, $R$, $C_i$, $i = 1, \dots I$ are viewed as a vector in $\mathbb{R}^{\vert \mathcal{S} \vert \vert \mathcal{A} \vert}$.
Since we define a mixture of policies in the trajectory level, the set of occupancy measures of policies in $\text{Conv}(\Pi)$ is just the convex hull of $\mathcal{V}$.
Since the above is an optimization problem in the space of $\mathbb{R}^{\vert S \vert \vert A \vert}$, strong duality holds if we assume the following Slater's condition.

\begin{assumption}[Slater's condition] \label{assumption:slater}
There exist a constant $\varphi > 0$ and a policy $\pi \in \Pi$ such that $J_{C_i}(\pi) \leq \tau_i - \frac{\varphi}{1 - \gamma}$ for all $i = 1, \dots, I$.
Assume $\varphi$ is known.
\end{assumption}

Slater's condition is a mild assumption commonly made for constrained RL \parencite{le_batch_2019,chen_primal-dual_2021,bai_achieving_2022,ding_natural_2020} for ensuring strong duality of the optimization problem.
It is mild because given the knowledge of the feasibility of the problem, we can guarantee that Slater's condition is met by slightly loosening the cost threshold.

\subsection{Offline Dataset}

In offline constrained RL, we assume access to an offline dataset $\mathcal{D} = \{(s_j, a_j, s_j') \}_{j = 1}^n$ where $(s_j, a_j)$ are generated i.i.d. from a data distribution $d^\mu$ induced by a behavior policy $\mu$ and $s_j' \sim P(\cdot~|~s_j, a_j)$.
Such an i.i.d. assumption on the offline dataset is commonly made in the offline RL literature \parencite{xie2021bellman,zhan_offline_2022,chen_offline_2022,zhu_importance_2023} to facilitate analysis of concentration bounds.
We assume the policy class $\Pi$ contains $\mu$.
We assume that the threshold $\bm\tau$ is chosen such that the optimization problem $\mathcal{P}(\bm\tau)$ is feasible.
However, we do not require the behavior policy $\mu$ to be feasible for $\mathcal{P}(\bm\tau)$.
To limit the distribution shift of policies from the data distribution, we make the following concentrability assumption,
where we use the notation $\Vert \cdot \Vert_{2, \mu} = \sqrt{\mathbb{E}_\mu[(\cdot)^2]}$.

\begin{assumption}[Concentrability] \label{assumption:concentrability}
For all $\pi \in \Pi$, we have $\Vert d^\pi / d^\mu \Vert_{2, \mu} \leq C_{\ell_2}$.
\end{assumption}

The concentrability assumption limits distribution shift of candidate policies from the data distribution. Specifically, the occupancy measure induced by a policy in $\Pi$ is covered by the data distribution $d^\mu$.
This assumption is weaker than the assumption made by \textcite{le_batch_2019}, who assume that the $\ell_\infty$ norm of the distribution shift of following any nonstationary policy that uses a policy in $\Pi$ every time step is bounded.

\subsection{Marginalized Importance Weight}

The notion of marginalized importance weight (MIW) is used extensively in the offline RL literature \parencite{xie_q_2020,chen_offline_2022,zhan_offline_2022,zhu_importance_2023,lee_coptidice_2022,lee_optidice_2021} to correct for the distribution mismatch between a policy $\pi$ and the behavior policy $\mu$.
It is defined as follows.
\begin{definition}[Marginalized importance weight]
For a policy $\pi$, we define the marginalized importance weight $w^\pi : \mathcal{S} \times \mathcal{A} \rightarrow \mathbb{R}^+$ as $w^\pi(s, a) \coloneqq \frac{d^\pi(s, a)}{d^\mu(s, a)}$.
\end{definition}
Immediately from the definition of MIW, we get the identity $\mathbb{E}_\pi[(\cdot)] = \mathbb{E}_\mu[w^\pi(s, a) (\cdot)]$, which we frequently use in the analysis.
We assume access to a function class $\mathcal{W}$ consisting of functions $w: \mathcal{S} \times \mathcal{A} \rightarrow \mathbb{R}_+$ that represents MIW with respect to the offline data distribution $d^\mu$.
We assume the following boundedness assumption on $\mathcal{W}$.

\begin{assumption}[Boundedness of $\mathcal{W}$] \label{assumption:w-boundedness}
Assume $\Vert w \Vert_\infty \leq C_\infty$ and $\Vert w \Vert_{2, \mu} \leq C_{\ell_2}$ for all $w \in \mathcal{W}$.
\end{assumption}

Denote by $\pi^\star$ an optimal policy of the optimization problem (\ref{eqn:opt}).
We assume that the MIW of $\pi^\star$ is realized by $\mathcal{W}$.

\begin{assumption}[Realizability of MIW] \label{assumption:miw-realizability}
Assume that $w^{\pi^\star} \in \mathcal{W}$ for an optimal policy $\pi^\star$.
\end{assumption}

The single-policy realizability of MIW that requires MIW of an optimal policy to be realizable by $\mathcal{W}$ is a weaker assumption than the all-policy realizability of MIW assumption required by \textcite{xie_q_2020}.
Compared to the set of assumptions made by \textcite{le_batch_2019}, we replace the strong Bellman completeness assumption with a single-policy realizability of MIW.

\section{ALGORITHM \& MAIN RESULTS}

In this section, we present our algorithm called Primal-Dual-Critic Algorithm (PDCA) and then present our main results on the the sample complexity bound.

\subsection{Primal-Dual Algorithm Structure}

The Lagrangian of the optimization problem $\mathcal{P}(\bm\tau)$ (Definition~\ref{def:opt}) is $L(\pi, \bm\lambda) = J_R(\pi) + \bm\lambda \cdot (\bm\tau - J_{\bm{C}}(\pi))$ where we use the notation $J_{\bm{C}}(\pi) = (J_{C_1}(\pi), \dots, J_{C_I}(\pi))$.
Our algorithm, like MBCL algorithm for offline constrained RL proposed by \textcite{le_batch_2019}, adopts the primal-dual algorithm structure that updates $\pi$ and $\bm\lambda$ alternatively.
The primal-dual algorithm structure can be seen as a sequential game of length $K$ between the $\pi$-player who tries to maximize $L(\cdot, \bm\lambda)$ and the $\lambda$-player who tries to minimize $L(\pi, \cdot)$.
Both players try to minimize their regrets against respective best actions in hindsight.

The key difference of our algorithm from MBCL is that in each round $k$, the $\pi$-player plays before the $\lambda$-player, while in MBCL, the order is reversed.
Since $\lambda$-player sees what $\pi$-player plays before playing, the $\lambda$-player can act greedily to minimize the regret.
On the other hand, the $\pi$-player has to use a no-regret policy optimization oracle, defined below, with a sublinear regret against adversarially chosen sequence of $\lambda$'s.

\begin{definition}[No-regret policy optimization oracle] \label{def:pi-player}
An algorithm is called a no-regret policy optimization oracle if for any sequence of functions $h_1, \dots, h_K : \mathcal{S} \times \mathcal{A} \rightarrow [-1, 1]$, the sequence of policies $\pi_1, \dots, \pi_K \in \Pi$ produced by the algorithm satisfies
$$
\frac{1}{K} \sum_{k = 1}^K \mathbb{E}_\pi [ h_k(s, \pi) - h_k(s, \pi_k) ] = \epsilon_{\text{opt}}(K)
$$
with high probability for any $\pi \in \mathrm{Conv}(\Pi)$ where $\epsilon_{\text{opt}}(K) \rightarrow 0$ as $K \rightarrow \infty$.
\end{definition}

A well-known instance of the oracle is the natural policy gradient algorithm \parencite{kakade_natural_2001} based on the updates $\pi_{k + 1}(a | s) \propto \pi_k(a | s) \exp(\eta h_k(s, a))$.

\subsection{Critics for Lagrangian}

We want to estimate the Lagrangian function $L(\pi, \bm\lambda) = J_R(\pi) + \bm\lambda \cdot (\bm\tau - J_{\bm{C}}(\pi))$ for all $\pi \in \Pi$ and $\bm\lambda \in B \bm\Delta^I$.
We use critics for $J_R(\pi)$ and $J_{\bm{C}}(\pi)$ that are inspired by the reward critic proposed by \textcite{zhu_importance_2023} for their actor-critic algorithm for offline unconstrained RL.
Our critic for $J_U(\pi)$ aims to solve
$$
\min_{f \in \mathcal{F}}~2 \mathcal{E}_\mu(\pi, f; U) \pm A_\mu(\pi, f)
$$
where the sign of $A_\mu$ is appropriately chosen and
\begin{align*}
\mathcal{E}_\mu(\pi, f; U) &\coloneqq \max_{w \in \mathcal{W}} \vert \mathbb{E}_\mu [ w(s, a)(f - \mathcal{T}^\pi_U f)(s, a)] \vert \\
A_\mu(\pi, f) &\coloneqq \mathbb{E}_\mu[f(s, \pi) - f(s, a)].
\end{align*}
Here, $\mathcal{T}_U^\pi : \mathbb{R}^{\mathcal{S} \times \mathcal{A}} \rightarrow \mathbb{R}^{\mathcal{S} \times \mathcal{A}}$ is the Bellman operator with $(\mathcal{T}_U^\pi f)(s, a) = U(s, a) + \gamma \mathbb{E}_{s' \sim P(\cdot~|~s, a)}[ f(s', \pi) ]$ and $f(s, \pi) = \sum_{a \in \mathcal{A}} \pi(a | s) f(s, a)$.
Minimizing the first term $\mathcal{E}_\mu$ in the objective function of the critics encourages Bellman-consistency.
The second term, with appropriate sign, facilitates the regret bound analysis for the $\pi$-player.
Since the data distribution of the behavior policy $d^\mu$ is unknown, we solve an empirical version $\min_{f \in \mathcal{F}} 2 \mathcal{E}_\mathcal{D}(\pi, f; U) + A_\mathcal{D}(\pi, f)$ where
\begin{align*}
\mathcal{E}_\mathcal{D}(\pi, f; U) &\coloneqq
\max_{w \in \mathcal{W}}\vert \mathbb{E}_\mathcal{D} [ w(s, a)(f(s, a) - U(s, a) - \gamma f(s', \pi))] \vert \\
A_\mathcal{D}(\pi, f) &\coloneqq \mathbb{E}_\mathcal{D}[f(s, \pi) - f(s, a)].
\end{align*}
Here,  $\mathbb{E}_\mathcal{D}[F(s, a, s')] = \frac{1}{\vert \mathcal{D} \vert} \sum_{(s, a, s') \in \mathcal{D}} F(s, a, s')$.
The critics for the reward value $J_R(\pi)$ and the cost value $J_{C_i}(\pi)$ are chosen to be
\begin{align}
\text{Reward critic:}&\quad \min_{f \in \mathcal{F}} 2 \mathcal{E}_\mathcal{D}(\pi, f; R) + A_\mathcal{D}(\pi, f) \label{eqn:reward-critic} \\
\text{Cost critic:}&\quad \min_{g \in \mathcal{F}} 2 \mathcal{E}_\mathcal{D}(\pi, g; C_i) - A_\mathcal{D}(\pi, g). \label{eqn:cost-critic}
\end{align}
Note that we use the same reward critic as the one used in \textcite{zhu_importance_2023}, but we negate the term $\mathcal{A}_\mathcal{D}$ for the cost critic.

\subsection{Cost Critic for $\lambda$-Player}

While the $\pi$-player optimizes for the Lagrangian estimated by critics (\ref{eqn:reward-critic}) and (\ref{eqn:cost-critic}) in the previous section, the $\lambda$-player optimizes for $\bm\lambda \cdot (\bm\tau - J_{\bm{C}}(\pi_k))$ estimated by an offline policy evaluation (OPE) oracle.

\begin{definition}[OPE oracle]
Let $\pi$ be a policy and $U$ a utility function.
Let $\mathcal{F}$ be a function class that contains the value function $Q_U^\pi$.
Suppose the dataset $\mathcal{D} = \{(s_j, a_j, s_j')\}_{j = 1}^n$ is generated with a behavior policy $\mu$ is a behavior policy with sufficient coverage such that $\Vert d^\pi / d^\mu \Vert_{2, \mu} \leq C_{\ell_2}$.
An algorithm that produces an estimate $h \in \mathbb{R}$ for $J_U(\pi)$ such that
$$
\vert h - J_U(\pi) \vert \leq \mathcal{O}\left(\frac{C_{\ell_2}}{1 - \gamma} \sqrt{\frac{\log(\vert \mathcal{F} \vert / \delta)}{n}}\right)
$$
with probability at least $1 - \delta$ is called an OPE oracle.
\end{definition}

An example of such an oracle is provided by \textcite{zanette_when_2022} whose algorithm produces an estimate with error decreasing at the required rate scaled by a constant factor called the incompleteness factor inherent to $\mathcal{F}$ that measures misalignment between $\mathcal{F}$ and $\mathcal{T}^\pi \mathcal{F}$.

\begin{remark}
We use a separate OPE cost critic for estimating the objective function of the $\lambda$-player instead of reusing the cost critic~(\ref{eqn:cost-critic}) used for estimating the Lagrangian for the $\pi$-player.
This is because the regrets for the two players take different forms.
As we show in Section~\ref{subsection:near-saddle-point}, the regret of $\pi$-player is the sum of $J_{R + \bm\lambda_k \cdot \bm{C}}(\pi^\star) - J_{R + \bm\lambda_k \cdot \bm{C}}(\pi_k)$ while the regret of $\lambda$-player is the sum of $(\bm\lambda_k - \bm\lambda^\star) \cdot (\bm\tau - J_{\bm{C}}(\pi_k))$.
Since the regret minimizing decision for $\lambda$-player depends on the sign of $\tau - J_{C_i}(\pi_k)$, we need to estimate $J_{C_i}(\pi_k)$, which involves calling an OPE oracle.
This is why we require concentrability of all policies in $\Pi$. We leave relaxing to a single-policy concentrability assumption as future work. \end{remark}

\begin{algorithm}
\KwInput{Dataset $\mathcal{D} = \{(s_j, a_j, s_j') \}_{j = 1}^n$, number of iterations $K$, cost thresholds $\bm\tau$, bound $B$, , no-regret policy optimization oracle PO, offline policy evaluation oracle OPE.}
\KwInit{$\pi_1$: uniform policy.}
\For{$k = 1, \dots, K$}{
  $f_k \leftarrow \argmin_{f \in \mathcal{F}} 2 \mathcal{E}_\mathcal{D}(\pi_k, f; R) + A_\mathcal{D}(\pi_k, f)$. \label{algline:reward-critic} \\
  $g_k^i \leftarrow \argmin_{g \in \mathcal{F}} 2 \mathcal{E}_\mathcal{D}(\pi_k, g; C_i) - A_\mathcal{D}(\pi_k, g)$ for all $i = 1, \dots, I$. \label{algline:cost-critic} \\
  $h_k^i \leftarrow \text{OPE}(\pi_k, C_i)$ for all $i = 1, \dots, I$. \\
  $\bm{\lambda}_k \leftarrow \argmin_{\bm\lambda \in B \bm\Delta^I} \bm\lambda \cdot (\bm\tau - \bm{h}_k)$. \\
  $\pi_{k + 1} \leftarrow \textsc{PO}(\pi_k, f_k + \bm\lambda_k \cdot (\bm\tau - \bm{g}_k))$. \\
}
\KwReturn{$\bar\pi = \text{Unif}(\pi_1, \dots, \pi_K)$}
\caption{\textbf{P}rimal-\textbf{D}ual-\textbf{C}ritic \textbf{A}lgorithm}
\label{alg:pdapc}
\end{algorithm}

\subsection{Proposed Algorithm}

We propose an algorithm called Primal-Dual-Critic Algorithm (PDCA) (Algorithm~\ref{alg:pdapc}) that uses critics for the Lagrangian function and no-regret policy optimization oracle for the $\pi$-player and a greedy $\lambda$-player.
The algorithm takes cost thresholds $\bm\tau$ as an input and runs a primal-dual algorithm on the estimate of Lagrangian $L(\pi, \bm\lambda) = J_R(\pi) + \bm\lambda \cdot (\bm\tau - J_{\bm{C}}(\pi))$.

The algorithm iterates for $K$ steps.
In each step $k$, the algorithm calculates the reward critic $f_k$ and the cost critics $\bm{g}_k = (g_k^1, \cdots, g_k^I)$ using the offline dataset $\mathcal{D}$.
Then, the $\pi$-player invokes a no-regret policy optimization oracle on the estimate $f_k + \bm\lambda_k \cdot (\bm\tau - \bm{g}_k)$.
The OPE oracle is used to estimate $J_{\bm{C}}(\pi_k)$ by $\bm{h}_k(s_0, \pi_k)$, and the $\lambda$-player chooses $\bm\lambda_k$ that minimizes $\bm\lambda \cdot (\bm\tau - \bm{h}_k)$.

After running the iterations, the algorithm returns the policy $\bar\pi = \text{Unif}(\pi_1, \dots, \pi_K)$, which is a uniform mixture of the policies $\pi_1, \dots, \pi_K$.
The uniform mixture policy initially samples a policy uniformly at random from $\{\pi_1, \dots, \pi_K \}$, and then follows the sampled policy for the entire trajectory.

\subsection{Main Results}

Our main theoretical result is a sample complexity bound of our algorithm called Primal-Dual-Critic Algorithm (PDCA) (Algorithm~\ref{alg:pdapc}) for finding a policy that satisfies the constraints approximately and is $\epsilon$-optimal with respect to the optimal policy $\pi^\star$ for $\mathcal{P}(\bm\tau)$.

\begin{restatable}{theorem}{mainthm} \label{thm:main-1}
Under assumptions~\ref{assumption:value-realizability},\ref{assumption:slater},\ref{assumption:concentrability},\ref{assumption:w-boundedness} and \ref{assumption:miw-realizability},
the policy $\bar\pi$ returned by the PDCA algorithm (Algorithm~\ref{alg:pdapc}) with the cost threshold $\bm\tau$, bound $B = 1 + \frac{1}{\varphi}$ and $K$ large enough, satisfies $J_{C_i}(\bar\pi) \leq \tau_i + \mathcal{O}(\epsilon)$ for all $i = 1, \dots, I$, and $J_R(\bar\pi) \geq J_R(\pi^\star) - \mathcal{O}(\epsilon)$ with probability at least $1 - \delta$ where $\pi^\star$ is an optimal policy for $\mathcal{P}(\bm\tau)$ as long as
$$
n \geq \Omega\left(\frac{(C_{\ell_2})^2 \log(I \vert \mathcal{F} \vert \vert \Pi \vert \vert \mathcal{W} \vert / \delta)}{(1 - \gamma)^{4} \varphi^2 \epsilon^2}\right).
$$
\end{restatable}

The sample complexity bound provided by \textcite{le_batch_2019} is $\Omega(\frac{C_\infty(\text{dim}_F + \log(I / \delta))}{(1 - \gamma)^{10} \epsilon^2})$ where $\text{dim}_F$ is a complexity measure of the function class they use for the Lagrangian function, which is analogous to the log cardinality term $\log \vert \mathcal{F} \vert$ in our finite function class setting.
Compared to their bound, our bound saves a factor of $\frac{1}{(1 - \gamma)^6}$ and depends on concentrability coefficient $C_{\ell_2}$ instead of $C_\infty$.
As \textcite{zhu_importance_2023} discuss, $(C_{\ell_2})^2 \leq C_\infty$ and $C_\infty$ can be arbitrarily larger than $(C_{\ell_2})^2$.
Also, our algorithm requires weaker assumptions.
While \textcite{le_batch_2019} require $\ell_\infty$ concentrability for all sequences of policies, we require $\ell_2$ concentrability for fixed policies.
While \textcite{le_batch_2019} require Bellman completeness for the value function class, we only require realizability.
The only additional assumption we need is the single-policy realizability of the marginalized importance weight.

With different choices of inputs to the PDCA algorithm, we get the following results.
See Appendix~\ref{appendix:main-results} for the formal statements and proofs.
\paragraph{Arbitrary Comparator Policy} (Theorem~\ref{thm:main-2}) Without the Slater's condition and the MIW realizability assumptions, running PDCA with the cost threshold  $\bm\tau$ and the bound $B = \frac{1}{(1 - \gamma) \epsilon}$ gives a policy $\bar\pi$ that is nearly feasible and satisfies near-optimality ($J_R(\bar\pi) \geq J_R(\pi_c) - \mathcal{O}(\epsilon)$) against any comparator policy $\pi_c \in \text{Conv}(\Pi)$ of which MIW is realizable by $\mathcal{W}$. However, the sample complexity bound is of $\mathcal{O}(1 / \epsilon^4)$.

\paragraph{Exact Feasibility} (Theorem~\ref{thm:main-3}) With the same sample complexity as in Theorem~\ref{thm:main-1}, running PDCA with a tightened cost threshold $\bm\tau - \eta \bm{1}$ where $\eta = \mathcal{O}(\epsilon)$ and the bound $B = \mathcal{O}(\frac{1}{\varphi})$ gives a policy $\bar\pi$ that is exactly feasible ($J_{\bm{C}}(\bar\pi) \leq \bm\tau$) and nearly optimal ($J_R(\bar\pi) \geq J_R(\pi^\star) - \mathcal{O}(\epsilon)$).

\section{ANALYSIS} \label{section:analysis}

In this section, we provide a proof sketch for Theorem~\ref{thm:main-1}.
We show in Section~\ref{subsection:near-saddle-point} that PDCA finds a near saddle point of the Lagrangian.
We show in Section~\ref{subsection:near-optimal} that the near saddle point approximately solves the optimization problem (\ref{eqn:opt}).
We use the notation $a \lesssim b$ to indicate $a \leq b + \zeta(n, K)$ and $a \approx b$ to indicate $a = b + \zeta(n, K)$ where $\zeta(n, K) \rightarrow 0$ as $n, K \rightarrow \infty$.

\subsection{PDCA Finds a Near Saddle Point} \label{subsection:near-saddle-point}

We show that PDCA run with thresholds $\bm\tau$ and bound $B$ finds a near saddle point of the Lagrangian $L(\pi, \bm\lambda) = J_R(\pi) + \bm\lambda \cdot (\bm\tau - J_{\bm{C}}(\pi))$: the policy $\bar\pi$ returned by PDCA and $\bar{\bm\lambda} = \frac{1}{K} \sum_{k = 1}^K \bm\lambda_k$ for sufficiently large $K$ satisfy
$$
L(\pi, \bar{\bm\lambda}) \leq L(\bar\pi, \bm\lambda) + \mathcal{O}(\epsilon_{\text{stat}})
$$
for all $\pi \in \text{Conv}(\Pi)$ with $w^\pi \in \mathcal{W}$ and $\bm\lambda \in B \bm\Delta^I$ where $\epsilon_{\text{stat}} = \mathcal{O}(1 / \sqrt{n})$ is a statistical error term for estimating $\mathcal{E}_\mu$ and $A_\mu$.
See Appendix~\ref{appendix:concentration} for the full analysis of $\epsilon_{\text{stat}}$.
To show that $(\bar\pi, \bar{\bm\lambda})$ is a near saddle point, we decompose $L(\pi, \bar{\bm\lambda}) - L(\bar\pi, \bm\lambda)$ into regrets of the $\pi$-player and the $\lambda$-player, and bound each regret separately as follows. See Appendix~\ref{appendix:saddle-point} for full proof.

\paragraph{Regret Bound for $\pi$-Player}
Regret of $\pi$-player $\frac{1}{K} \sum_{k = 1}^K L(\pi, \bm\lambda_k) - \frac{1}{K} \sum_{k = 1}^K L(\pi_k, \bm\lambda_k)$ simplifies to $\frac{1}{K} \sum_{k = 1}^K (J_R(\pi) - J_R(\pi_k)) + \frac{1}{K} \sum_{k = 1}^K \bm\lambda_k \cdot (J_{\bm{C}}(\pi_k) - J_{\bm{C}}(\pi))$.
Decomposing $J_R(\pi) - J_R(\pi_k)$ by performance difference lemma (Lemma 12 in \textcite{cheng_adversarially_2022}) and using the properties of the critics give
$$
(1 - \gamma)(J_R(\pi) - J_R(\pi_k)) \lesssim
\mathbb{E}_\pi[f_k(s, \pi) - f_k(s, \pi_k)]
$$
which shows the performance difference with respect to the reward function $R$ can be upper bounded by the difference of the reward critic.
Similarly, we have
$$
(1 - \gamma)(J_{C_i}(\pi_k) - J_{C_i}(\pi)) \lesssim \mathbb{E}_\pi[g_k^i(s, \pi_k) - g_k^i(s, \pi)],
$$
for all $i = 1, \dots, I$, and it follows that
\begin{align*}
\textstyle \frac{1}{K} &\textstyle\sum_{k = 1}^K L(\pi, \bm\lambda_k) - \frac{1}{K} \textstyle\sum_{k = 1}^K L(\pi_k, \bm\lambda_k) \\
&\lesssim
\textstyle\frac{1}{K} \textstyle\sum_{k = 1}^K \mathbb{E}_\pi [z_k(s, \pi) - z_k(s, \pi_k)]
\leq \frac{1}{K} \epsilon_{\text{opt}}(K)
\end{align*}
where $z_k = f_k + \bm\lambda_k \cdot (\bm\tau - \bm{g}_k)$ is the input to policy optimization oracle (Definition~\ref{def:pi-player}) used by $\pi$-player; the last inequality is by the property of the oracle.

\paragraph{Regret Bound for $\lambda$-Player}
The regret of the $\lambda$-player $\frac{1}{K} \sum_{k = 1}^K L(\pi_k, \bm\lambda_k) - L(\bar\pi, \bm\lambda)$ can be simplified to $\frac{1}{K} \sum_{k = 1}^K (\bm\lambda_k - \bm\lambda) \cdot (\bm\tau - J_{\bm{C}}(\pi_k))$.
Recall that PDCA calls OPE oracle to estimate $J_{C_i}(\pi_k) \approx h_k^i$.
Since the $\lambda$-player greedily chooses $\bm\lambda_k \in B \bm\Delta^I$ that minimizes $\bm\lambda \cdot (\bm\tau - \bm{h}_k)$, we have for all $\bm\lambda \in B \bm\Delta^I$ that
$$
\textstyle\frac{1}{K} \textstyle\sum_{k = 1}^K (\bm\lambda_k - \bm\lambda) \cdot (\bm\tau - J_{\bm{C}}(\pi_k))
\approx
\textstyle\frac{1}{K} \textstyle\sum_{k = 1}^K (\bm\lambda_k - \bm\lambda)\cdot(\bm\tau - \bm{h}_k)
\leq
0.
$$

\subsection{A Near Saddle Point Nearly solves OPT} \label{subsection:near-optimal}

We can show that if the Slater's condition (Assumption~\ref{assumption:slater}) holds, then a near saddle point
 $(\bar\pi, \bar{\bm\lambda})$ of $L(\cdot, \cdot)$ that satisfies $L(\pi, \bar{\bm\lambda}) \leq L(\bar\pi, \bm\lambda)$ for all $\pi \in \text{Conv}(\Pi)$ with $w^\pi \in \mathcal{W}$ and $\bm\lambda \in B \bm\Delta^I$, then
\begin{align*}
J_R(\bar\pi) &\geq J_R(\pi^\star) - \xi \\
J_{C_i}(\bar\pi) &\leq \tau_i + \textstyle \frac{\xi}{B - 1 / \varphi}, \quad i = 1, \dots, I
\end{align*}
where $\pi^\star \in \text{Conv}(\Pi)$ is the optimal policy for $\mathcal{P}(\bm\tau)$.
See Appendix~\ref{appendix:saddle-point-property} for the proof of the above result.
Combining the results in Section~\ref{subsection:near-saddle-point} and Section~\ref{subsection:near-optimal}, Theorem~\ref{thm:main-1} follows.
See Appendix~\ref{appendix:main-results} for the full proof.

\section{EXPERIMENTS} \label{section:experiments}

To demonstrate the empirical performance of our algorithm, we compare with MBCL \parencite{le_batch_2019}, the only previous work with provable guarantees, and the following practical algorithms for offline constrained RL: COptiDICE~\parencite{lee_coptidice_2022}, CDT~\parencite{liu2023constrained}, CQP~\parencite{xu2022constraints}.
For comparing with MBCL, we use tabular setting since MBCL can only run with discrete action set.
For computational efficiency in solving the min-max optimization problems in Line \ref{algline:reward-critic},\ref{algline:cost-critic} of Algorithm~\ref{alg:pdapc}, we take $\mathcal{W} = [0, C_\infty]^{\mathcal{S} \times \mathcal{A}}$ where the bound $C_\infty$ is treated as a hyperparameter.
Such $\mathcal{W}$ reduces $\mathcal{E}_\mathcal{D}(\pi, f; U)$ to
\begin{align*}
\mathcal{E}_\mathcal{D}(\pi, f; U)
&= \max_{w \in \mathcal{W}} \vert \mathbb{E}_\mathcal{D} [ w(s, a) (f(s, a) - U(s, a) - \gamma f(s', \pi))] \vert \\
&=
C_\infty \max \{
\mathbb{E}_\mathcal{D} [ (f(s, a) - U(s, a) - \gamma f(s', \pi))_+],
\mathbb{E}_\mathcal{D} [ (U(s, a) + \gamma f(s', \pi) - f(s, a))_+]
\}.
\end{align*}
Following experimental settings of previous works, we use a single cost signal ($I = 1$) for all experiments.
For experimental details, see Appendix~\ref{appendix:experiments}.

\begin{figure}[t]
    \centering
    \includegraphics[width=0.7\columnwidth]{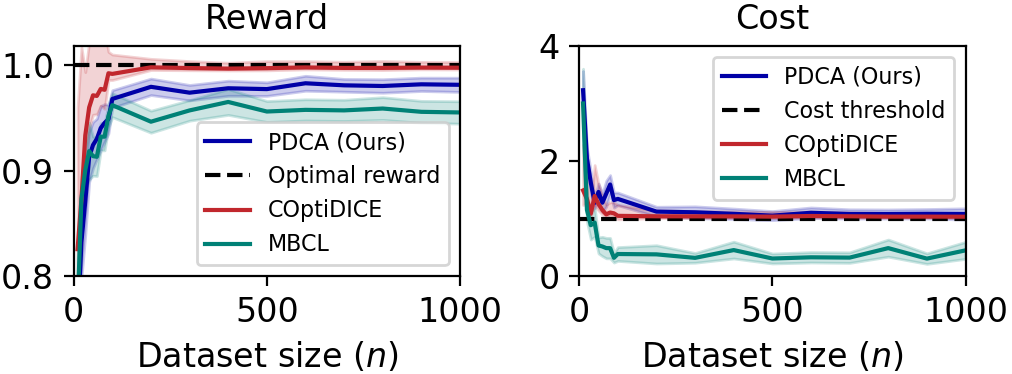}
    \caption{Tabular CMDP experiment}
    \label{fig:tabular}
\end{figure}

\subsection{Tabular Experiments}

Following \textcite{lee_coptidice_2022}, we randomly generate tabular CMDP with 10 states and 5 actions and prepare an offline dataset using a data distribution induced by a mixture of uniform policy and the optimal policy.
We compare the performance of PDCA to MBCL and COptiDICE on datasets of varying sizes.
For each dataset size, we repeat the experiments 10 times and report the average of the reward value and the cost value.
Figure~\ref{fig:tabular} shows the result.
The shaded region indicates the standard error.
Overall, PDCA outperforms MBCL and is comparable to COptiDICE.

\subsection{Real-World RL Benchmark Experiments}

\begin{algorithm}[t]
\KwInput{Dataset $\mathcal{D} = \{ (s_j, a_j, s_j') \}_{j = 1}^n$}
\KwInit{Network $f_\theta$ for reward critic; $g^i_{\theta^i}$, $i = 1, \dots, I$ for cost critics; $h^i_{\vartheta^i}$, $i = 1, \dots, I$ for OPE. Policy network $\pi_\psi$.}
\For{$k = 1, 2, \dots, K$}{
Sample a minibatch $\mathcal{D}_{\text{mini}}$ from dataset $\mathcal{D}$. \\
\tcp{Update critics}
$\ell_{\text{reward}}(\theta) = 2 \mathcal{E}_{\mathcal{D}_{\text{mini}}} (f_\theta, \pi_\psi) + A_{\mathcal{D}_{\text{mini}}}(f_\theta, \pi_\psi)$. \\
$\ell_{\text{cost}}(\theta^i) = 2 \mathcal{E}_{\mathcal{D}_{\text{mini}}} (g^i_{\theta^i}, \pi_\psi) - A_{\mathcal{D}_{\text{mini}}}(g^i_{\theta^i}, \pi_\psi)$ for $i = 1, \dots, I$. \\
$\ell_{\text{ope}}(\theta^i) = \mathcal{E}_{\mathcal{D}_{\text{mini}}} (g^i_{\theta^i}, \pi_\psi)$ for $i = 1, \dots I$. \\
$\theta \leftarrow \textsc{Adam}(\pi_\psi, \nabla \ell_{\text{reward}}(\theta), \eta_{\text{fast}})$. \\
$\theta^i \leftarrow \textsc{Adam}(\pi_\psi, \nabla \ell_{\text{cost}}(\theta^i), \eta_{\text{fast}})$, $i = 1, \dots, I$. \\
$\vartheta^i \leftarrow \textsc{Adam}(\pi_\psi, \nabla \ell_{\text{ope}}(\vartheta^i), \eta_{\text{fast}})$, $i = 1, \dots, I$. \\
\tcp{Update $\pi$.}
$\ell_{\text{actor}}(\psi) = -A_{\mathcal{D}_{\text{mini}}}(f_\theta + \sum_{i = 1}^I \lambda_i(\tau_i - g^i_{\theta^i}), \pi_\psi)$. \\
$\psi \leftarrow \textsc{Adam}(\ell_{\text{actor}}, \eta_{\text{slow}})$. \\
\tcp{Update $\lambda$.}
$z_i \leftarrow \tau_i - h^i_{\theta^i}(s_0, \pi_\psi)$, for $i = 1, \dots, I$ \\
$\lambda_i \leftarrow B$ if $z_i < 0$ otherwise $\lambda_i \leftarrow 0$, $i = 1, \dots, I$.
}
\caption{Practical Version of PDCA}
\label{alg:practical}
\end{algorithm}

We follow the experimental setup in \textcite{lee_coptidice_2022} and run the algorithms on 4 environments provided in the Real-World RL (RWRL) suite  \parencite{dulacarnold2020realworldrlempirical}.
For the benchmark experiments, we use a practical version of PDCA shown in Algorithm~\ref{alg:practical}.
We parameterize the function class $\mathcal{F}$ with neural networks.
The reward critic uses a neural network $f_\theta$ parameterized by $\theta$ and each cost critic for the cost $C_i$ uses a neural network $g^i_{\theta^i}$ parameterized by $\theta^i$.
For solving the optimization problems, the reward critic uses stochastic gradient descent algorithm with a learning rate $\eta_{\text{fast}}$ on the loss $2 \mathcal{E}_\mathcal{D}(f_\theta, \pi) + A_\mathcal{D}(f_\theta, \pi)$.
Similarly, the cost critic uses stochastic gradient descent algorithm with the same learning rate $\eta_{\text{fast}}$ on the loss $2 \mathcal{E}_\mathcal{D}(f_\theta, \pi) - A_\mathcal{D}(f_\theta, \pi)$.
Following the practical version of no-regret policy optimization oracle implemented by \textcite{cheng_adversarially_2022}, we use a policy network to parameterize $\Pi$.
The $\pi$-player uses a neural network $\pi_\psi$ parameterized by $\psi$ and use a stochastic gradient descent algorithm on the loss $-A_\mathcal{D}(f_\theta + \bm\lambda \cdot (\bm\tau - \bm{g}_{\bm\theta}), \pi_\psi)$.
For the OPE oracle, we use a neural network $h^i_{\vartheta^i}$ parameterized by $\vartheta^i$ and use a stochastic gradient descent algorithm on the loss $\mathcal{E}_{\mathcal{D}}(h^i_{\vartheta^i}, \pi)$ with learning rate $\eta_{\text{fast}}$.
The $\lambda$-player acts greedily and chooses $\bm\lambda \in B \bm\Delta^I$ that minimizes $\bm\lambda \cdot (\bm\tau - \bm{h}_\vartheta(s_0, \pi_\psi))$.
See Appendix~\ref{appendix:rwrl} for hyperparameter tuning details.

\paragraph{Environments}

We run experiments on four environments provided in the Real-World RL (RWRL) Benchmark suite \parencite{dulacarnold2020realworldrlempirical} used by \textcite{lee_coptidice_2022}: Cartpole, Walker, Quadruped, and Humanoid.
Following \textcite{lee_coptidice_2022}, for each environment, we choose the most challenging safety condition among the multiple safety conditions provided by RWRL suite.
We give the cost of 1 if the safety condition is violated at each time step.
The thresholds on the expected discounted cumulative costs are
 0.05 for Cartpole and Walker, and 0.01 for Quadruped and Humanoid.
We follow the same safety coefficient parameters (difficulty levels provided by RWRL suite) used by \textcite{lee_coptidice_2022}: for Cartpole and Walker we use 0.3, and for Quadruped and Humanoid we use 0.5.

\paragraph{Offline Dataset Generation}

Since RWRL suite does not provide an offline dataset we generate one for each environment by a policy trained by an online RL algorithm using a reward function penalized by cost function, $R - \lambda C$, where we vary $\lambda$.
Specifically, for each environment, we choose three different $\lambda$ values and for each $\lambda$, we run the soft actor-critic algorithm (SAC) \parencite{haarnoja_soft_2018} with the reward function $R - \lambda C$.
The SAC algorithm is run for 1,000,000 steps.
For each policy trained with different $\lambda$ values, we generate 1,000 trajectories.
During trajectory generation, actions are perturbed with Gaussian noise with mean=0 and std=0.15.
The three sets of trajectories, one for each $\lambda$, are mixed to form an offline dataset consisting of 3,000 trajectories.
For the $\lambda$ values, we use $0.3,0.8,1$ for Cartpole, $1,1.8,2$ for Walker, $0,0.1,0.5$ for Quadruped, $0,0.4,0.5$ for Humanoid.

\paragraph{Evaluation}

Every 1000 iterations, we evaluate the policy by running it on the environment online and recording the trajectories 5 times.
We report the average discounted cumulative reward and cost and their standard errors.
Figure~\ref{fig:benchmark} shows comparison of the performance of our algorithm and COptiDICE on 4 RWRL environments.
The black dotted horizontal line indicates the cost threshold.
The blue and red lines indicate the cumulative cost and reward for PDCA and COptiDICE respectively.
For the Cartpole environment, PDCA outperforms COptiDICE.
For the Walker environment, PDCA requires more iterations to converge and the performance is comparable to COptiDICE.
For the Quadruped environment, the cumulative rewards for PDCA and COptiDICE are comparable, but COptiDICE tends to produce a policy that is constraint-violating.
For the Humanoid environment, PDCA is comparable to COptiDICE.

\begin{figure}[t]
\centering
\includegraphics[width=1\linewidth]{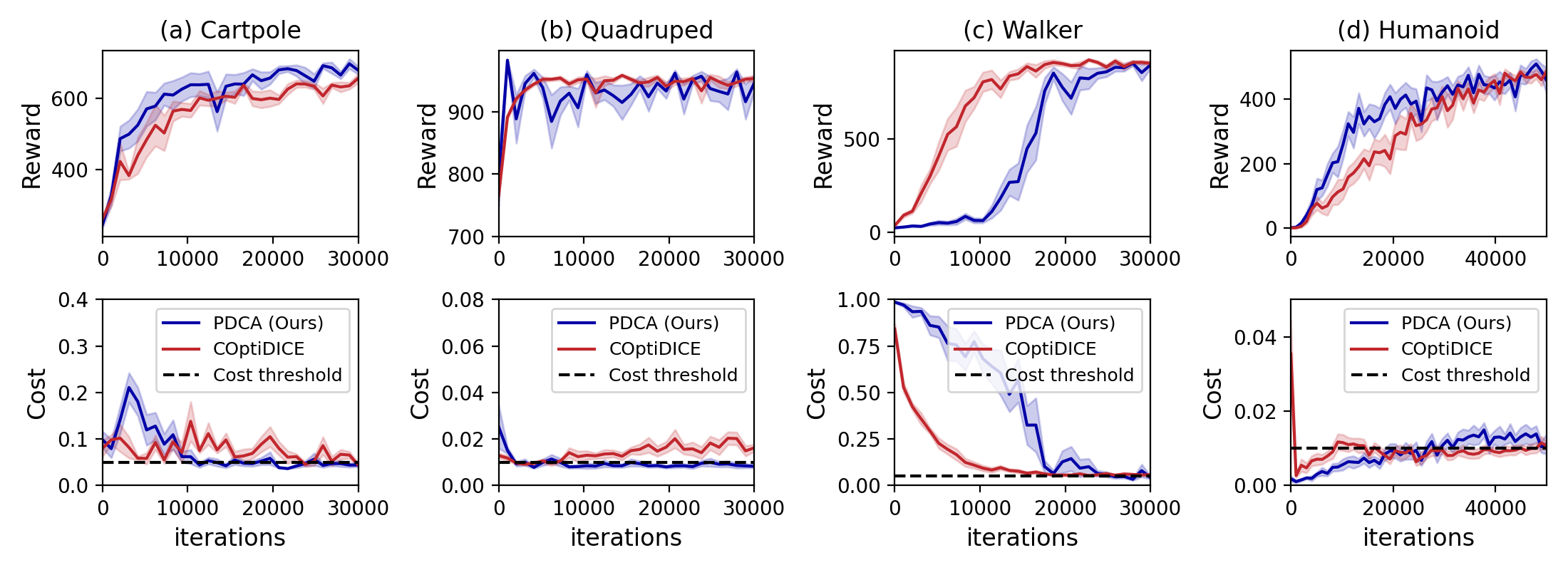}%
\caption{RWRL Benchmark Experiments}
\label{fig:benchmark}
\end{figure}

\subsection{Safety Gym Benchmark Experiments}

We run PDCA on the bullet safety gym~\parencite{Gronauer2022BulletSafetyGym} with offline datasets provided by \textcite{liu2023datasets} and compare the performance of PDCA to CDT~\parencite{liu2023constrained}, CPQ~\parencite{xu2022constraints} and COptiDICE~\parencite{lee_coptidice_2022}.
See Appendix~\ref{appendix:bullet} for the details of the offline datasets and hyperparameter tuning procedure.

\paragraph{Environments}

The Bullet Safety Gym \parencite{Gronauer2022BulletSafetyGym} provides environments based on physics simulator where the agent can move around the physical environment scattered with obstacles.
The layout of the obstacles is not fixed and randomly generated in each episode.
For our benchmark experiments, we use the three different agents: ball that can move freely on a plane and is controlled by a two-dimensional force vector; car that can control wheel velocities and steering angle; ant that is quadrupedal composed of nine rigid bodies with each leg controlled by two actuators.
We use two different tasks. The circle task encourages the agent to move on a circle. The reward signal depends on the speed of the agent and the proximity of the agent to the boundary. Costs are incurred when the agent leaves the circle.
The run task rewards the agent for running through an avenue between two safety boundaries. The agent incurs costs when exceeding speed limit.

\paragraph{Evaluation}

Following the evaluation protocol by \textcite{liu2023datasets}, we run with cost thresholds set to 10, 20, 40.
For each of the cost threshold, we run with 3 random seeds.
We report the average performance across the 9 runs for each environment.
To approximate the uniform mixture of historical policies produced by PDCA,  we take policies every 2500 iterations and report the average of the performance of the policies.
The performance of each policy is measured by running the policy on the corresponding environment for 20 episodes and taking the average of the discounted cumulative reward/cost.
When reporting, the cost value is normalized so that the cost threshold is scaled to 1.
See Table~\ref{table:safety-gym} for the results.
The column \reward{\textbf{R}} is the reward and \cost{\textbf{C}} the cost averaged over three random seeds and three cost thresholds.
Cost is normalized such that the threshold value is 1.
Boldfaced numbers indicate cost values that exceed the threshold.
The performance of PDCA is generally not dominated by CPQ and COptiDICE in the sense that none of the algorithms outperform PDCA in terms of both reward and cost violation.
The transformer based algorithm CDT generally outperforms other algorithms. We believe that this is because CDT learns non-Markovian policies, which may be better suited for the benchmark environments.

\begin{table}[h]
\caption{Safety Gym Results (Cost Threshold = 1.00)}
\label{table:safety-gym}
\centering
\begin{tabular}{ccccccccc}
 \toprule
 \multirow{2}{*}{\textbf{Task}} & \multicolumn{2}{c}{\textbf{CDT}} & \multicolumn{2}{c}{\textbf{CPQ}} & \multicolumn{2}{c}{\textbf{COptiDICE}} & \multicolumn{2}{c}{\textbf{PDCA}} \\
 \cline{2-9}
 & \reward{\textbf{R}} & \cost{\textbf{C}} & \reward{\textbf{R}} & \cost{\textbf{C}} &\reward{\textbf{R}} & \cost{\textbf{C}} &\reward{\textbf{R}} & \cost{\textbf{C}} \\
 \midrule
 AntCircle & \reward{0.54} & \cost{\bftab{1.78}} & \reward{0.00} & \cost{0.00} & \reward{0.17} & \cost{\bftab{5.04}} & \reward{0.22} & \cost{\bftab{3.53}} \\
 AntRun & \reward{0.72} & \cost{0.91} & \reward{0.03} & \cost{0.02} & \reward{0.61} & \cost{0.94} & \reward{0.28} & \cost{0.93} \\
 BallCircle & \reward{0.77} & \cost{\bftab{1.07}} & \reward{0.64} & \cost{0.76} & \reward{0.70} & \cost{\bftab{2.61}} & \reward{0.63} & \cost{\bftab{2.29}} \\
 BallRun & \reward{0.39} & \cost{\bftab{1.16}} & \reward{0.22} & \cost{\bftab{1.27}} & \reward{0.59} & \cost{\bftab{3.52}} & \reward{0.55} & \cost{\bftab{3.38}} \\
 CarCircle & \reward{0.75} & \cost{0.95} & \reward{0.71} & \cost{0.33} & \reward{0.49} & \cost{\bftab{3.14}} & \reward{0.22} & \cost{\bftab{2.42}} \\
 \bottomrule
\end{tabular}
\end{table}

\section{CONCLUSION}

We propose a primal-dual algorithm PDCA for offline constrained RL with function approximation.
PDCA is sample-efficient under concentrability, value function realizability and MIW realizability assumptions, which relaxes Bellman completeness assumption required by previous work.
PDCA requires all-policy concentrability only to guarantee the concentration bound on the estimates returned by OPE.
Relaxing this to the single-policy concentrability assumption is an interesting future work that will likely require using pessimistic estimates for the costs and modifying the strategy of the $\lambda$-player to work with pessimistic estimates.

\printbibliography

@article{munos_finite-time_2008,
	title = {Finite-{Time} {Bounds} for {Fitted} {Value} {Iteration}},
	volume = {9},
	issn = {1533-7928},
	url = {http://jmlr.org/papers/v9/munos08a.html},
	number = {27},
	urldate = {2023-05-01},
	journal = {Journal of Machine Learning Research},
	author = {Munos, Rémi and Szepesvári, Csaba},
	year = {2008},
	pages = {815--857},
}

@inproceedings{munos_error_2005,
	address = {Pittsburgh, Pennsylvania},
	series = {{AAAI}'05},
	title = {Error bounds for approximate value iteration},
	isbn = {978-1-57735-236-5},
	urldate = {2023-05-01},
	booktitle = {Proceedings of the 20th national conference on {Artificial} intelligence - {Volume} 2},
	publisher = {AAAI Press},
	author = {Munos, Rémi},
	month = jul,
	year = {2005},
	pages = {1006--1011},
}

@inproceedings{munos_error_2003,
	address = {Washington, DC, USA},
	series = {{ICML}'03},
	title = {Error bounds for approximate policy iteration},
	isbn = {978-1-57735-189-4},
	urldate = {2023-05-01},
	booktitle = {Proceedings of the {Twentieth} {International} {Conference} on {International} {Conference} on {Machine} {Learning}},
	publisher = {AAAI Press},
	author = {Munos, Rémi},
	month = aug,
	year = {2003},
	pages = {560--567},
}

@article{levine_learning_2018,
	title = {Learning hand-eye coordination for robotic grasping with deep learning and large-scale data collection},
	volume = {37},
	issn = {0278-3649, 1741-3176},
	url = {http://journals.sagepub.com/doi/10.1177/0278364917710318},
	doi = {10.1177/0278364917710318},
	language = {en},
	number = {4-5},
	urldate = {2023-05-12},
	journal = {The International Journal of Robotics Research},
	author = {Levine, Sergey and Pastor, Peter and Krizhevsky, Alex and Ibarz, Julian and Quillen, Deirdre},
	month = apr,
	year = {2018},
	pages = {421--436},
}

@misc{levine_offline_2020,
	title = {Offline {Reinforcement} {Learning}: {Tutorial}, {Review}, and {Perspectives} on {Open} {Problems}},
	shorttitle = {Offline {Reinforcement} {Learning}},
	url = {http://arxiv.org/abs/2005.01643},
	doi = {10.48550/arXiv.2005.01643},
	urldate = {2022-09-29},
	publisher = {arXiv},
	author = {Levine, Sergey and Kumar, Aviral and Tucker, George and Fu, Justin},
	month = nov,
	year = {2020},
	note = {arXiv:2005.01643 [cs, stat]},
}

@inproceedings{lee_coptidice_2022,
	title = {{COptiDICE}: {Offline} {Constrained} {Reinforcement} {Learning} via {Stationary} {Distribution} {Correction} {Estimation}},
	shorttitle = {{COptiDICE}},
	url = {https://openreview.net/forum?id=FLA55mBee6Q},
	language = {en},
	urldate = {2022-08-07},
	author = {Lee, Jongmin and Paduraru, Cosmin and Mankowitz, Daniel J. and Heess, Nicolas and Precup, Doina and Kim, Kee-Eung and Guez, Arthur},
	month = mar,
	year = {2022},
}

@inproceedings{lee_optidice_2021,
	title = {{OptiDICE}: {Offline} {Policy} {Optimization} via {Stationary} {Distribution} {Correction} {Estimation}},
	shorttitle = {{OptiDICE}},
	url = {https://proceedings.mlr.press/v139/lee21f.html},
	language = {en},
	urldate = {2022-09-02},
	booktitle = {Proceedings of the 38th {International} {Conference} on {Machine} {Learning}},
	publisher = {PMLR},
	author = {Lee, Jongmin and Jeon, Wonseok and Lee, Byungjun and Pineau, Joelle and Kim, Kee-Eung},
	month = jul,
	year = {2021},
	note = {ISSN: 2640-3498},
	pages = {6120--6130},
}

@inproceedings{le_batch_2019,
	title = {Batch {Policy} {Learning} under {Constraints}},
	url = {https://proceedings.mlr.press/v97/le19a.html},
	language = {en},
	urldate = {2022-07-30},
	booktitle = {Proceedings of the 36th {International} {Conference} on {Machine} {Learning}},
	publisher = {PMLR},
	author = {Le, Hoang and Voloshin, Cameron and Yue, Yisong},
	month = may,
	year = {2019},
	note = {ISSN: 2640-3498},
	pages = {3703--3712},
}

@inproceedings{kumar_workflow_2021,
	title = {A {Workflow} for {Offline} {Model}-{Free} {Robotic} {Reinforcement} {Learning}},
	url = {https://openreview.net/forum?id=fy4ZBWxYbIo},
	language = {en},
	urldate = {2023-05-12},
	author = {Kumar, Aviral and Singh, Anikait and Tian, Stephen and Finn, Chelsea and Levine, Sergey},
	month = sep,
	year = {2021},
}

@inproceedings{kakade_natural_2001,
	title = {A {Natural} {Policy} {Gradient}},
	volume = {14},
	url = {https://proceedings.neurips.cc/paper/2001/hash/4b86abe48d358ecf194c56c69108433e-Abstract.html},
	urldate = {2023-03-16},
	booktitle = {Advances in {Neural} {Information} {Processing} {Systems}},
	publisher = {MIT Press},
	author = {Kakade, Sham M},
	year = {2001},
}

@inproceedings{kakade_approximately_2002,
	address = {San Francisco, CA, USA},
	series = {{ICML} '02},
	title = {Approximately {Optimal} {Approximate} {Reinforcement} {Learning}},
	isbn = {978-1-55860-873-3},
	urldate = {2023-03-16},
	booktitle = {Proceedings of the {Nineteenth} {International} {Conference} on {Machine} {Learning}},
	publisher = {Morgan Kaufmann Publishers Inc.},
	author = {Kakade, Sham and Langford, John},
	month = jul,
	year = {2002},
	pages = {267--274},
}

@inproceedings{haarnoja_soft_2018,
	title = {Soft {Actor}-{Critic}: {Off}-{Policy} {Maximum} {Entropy} {Deep} {Reinforcement} {Learning} with a {Stochastic} {Actor}},
	shorttitle = {Soft {Actor}-{Critic}},
	url = {https://proceedings.mlr.press/v80/haarnoja18b.html},
	language = {en},
	urldate = {2023-05-17},
	booktitle = {Proceedings of the 35th {International} {Conference} on {Machine} {Learning}},
	publisher = {PMLR},
	author = {Haarnoja, Tuomas and Zhou, Aurick and Abbeel, Pieter and Levine, Sergey},
	month = jul,
	year = {2018},
	note = {ISSN: 2640-3498},
	pages = {1861--1870},
}

@inproceedings{gu_deep_2017,
	title = {Deep reinforcement learning for robotic manipulation with asynchronous off-policy updates},
	doi = {10.1109/ICRA.2017.7989385},
	booktitle = {2017 {IEEE} {International} {Conference} on {Robotics} and {Automation} ({ICRA})},
	author = {Gu, Shixiang and Holly, Ethan and Lillicrap, Timothy and Levine, Sergey},
	month = may,
	year = {2017},
	pages = {3389--3396},
}

@misc{foster_offline_2022,
	title = {Offline {Reinforcement} {Learning}: {Fundamental} {Barriers} for {Value} {Function} {Approximation}},
	shorttitle = {Offline {Reinforcement} {Learning}},
	url = {http://arxiv.org/abs/2111.10919},
	doi = {10.48550/arXiv.2111.10919},
	urldate = {2023-01-01},
	publisher = {arXiv},
	author = {Foster, Dylan J. and Krishnamurthy, Akshay and Simchi-Levi, David and Xu, Yunzong},
	month = aug,
	year = {2022},
	note = {arXiv:2111.10919 [cs, stat]},
}

@inproceedings{ding_natural_2020,
	title = {Natural {Policy} {Gradient} {Primal}-{Dual} {Method} for {Constrained} {Markov} {Decision} {Processes}},
	volume = {33},
	url = {https://proceedings.neurips.cc/paper/2020/hash/5f7695debd8cde8db5abcb9f161b49ea-Abstract.html},
	urldate = {2023-04-16},
	booktitle = {Advances in {Neural} {Information} {Processing} {Systems}},
	publisher = {Curran Associates, Inc.},
	author = {Ding, Dongsheng and Zhang, Kaiqing and Basar, Tamer and Jovanovic, Mihailo},
	year = {2020},
	pages = {8378--8390},
}

@misc{chen_primal-dual_2021,
	title = {A {Primal}-{Dual} {Approach} to {Constrained} {Markov} {Decision} {Processes}},
	url = {http://arxiv.org/abs/2101.10895},
	doi = {10.48550/arXiv.2101.10895},
	urldate = {2023-04-16},
	publisher = {arXiv},
	author = {Chen, Yi and Dong, Jing and Wang, Zhaoran},
	month = jan,
	year = {2021},
	note = {arXiv:2101.10895 [math]},
}

@inproceedings{cheng_adversarially_2022,
	title = {Adversarially {Trained} {Actor} {Critic} for {Offline} {Reinforcement} {Learning}},
	url = {https://proceedings.mlr.press/v162/cheng22b.html},
	language = {en},
	urldate = {2023-02-13},
	booktitle = {Proceedings of the 39th {International} {Conference} on {Machine} {Learning}},
	publisher = {PMLR},
	author = {Cheng, Ching-An and Xie, Tengyang and Jiang, Nan and Agarwal, Alekh},
	month = jun,
	year = {2022},
	note = {ISSN: 2640-3498},
	pages = {3852--3878},
}

@inproceedings{chen_offline_2022,
	title = {Offline reinforcement learning under value and density-ratio realizability: {The} power of gaps},
	shorttitle = {Offline reinforcement learning under value and density-ratio realizability},
	url = {https://proceedings.mlr.press/v180/chen22g.html},
	language = {en},
	urldate = {2023-05-02},
	booktitle = {Proceedings of the {Thirty}-{Eighth} {Conference} on {Uncertainty} in {Artificial} {Intelligence}},
	publisher = {PMLR},
	author = {Chen, Jinglin and Jiang, Nan},
	month = aug,
	year = {2022},
	note = {ISSN: 2640-3498},
	pages = {378--388},
}

@inproceedings{chen_information-theoretic_2019,
	title = {Information-{Theoretic} {Considerations} in {Batch} {Reinforcement} {Learning}},
	url = {https://proceedings.mlr.press/v97/chen19e.html},
	language = {en},
	urldate = {2023-01-04},
	booktitle = {Proceedings of the 36th {International} {Conference} on {Machine} {Learning}},
	publisher = {PMLR},
	author = {Chen, Jinglin and Jiang, Nan},
	month = may,
	year = {2019},
	note = {ISSN: 2640-3498},
	pages = {1042--1051},
}

@book{boyd_convex_2004,
	address = {Cambridge, UK ; New York},
	title = {Convex optimization},
	isbn = {978-0-521-83378-3},
	publisher = {Cambridge University Press},
	author = {Boyd, Stephen P. and Vandenberghe, Lieven},
	year = {2004},
}

@article{bai_achieving_2022,
	title = {Achieving {Zero} {Constraint} {Violation} for {Constrained} {Reinforcement} {Learning} via {Primal}-{Dual} {Approach}},
	volume = {36},
	copyright = {Copyright (c) 2022 Association for the Advancement of Artificial Intelligence},
	issn = {2374-3468},
	url = {https://ojs.aaai.org/index.php/AAAI/article/view/20281},
	doi = {10.1609/aaai.v36i4.20281},
	language = {en},
	number = {4},
	urldate = {2022-07-31},
	journal = {Proceedings of the AAAI Conference on Artificial Intelligence},
	author = {Bai, Qinbo and Bedi, Amrit Singh and Agarwal, Mridul and Koppel, Alec and Aggarwal, Vaneet},
	month = jun,
	year = {2022},
	note = {Number: 4},
	pages = {3682--3689},
}

@article{antos_learning_2008,
	title = {Learning near-optimal policies with {Bellman}-residual minimization based fitted policy iteration and a single sample path},
	volume = {71},
	issn = {1573-0565},
	url = {https://doi.org/10.1007/s10994-007-5038-2},
	doi = {10.1007/s10994-007-5038-2},
	language = {en},
	number = {1},
	urldate = {2023-05-17},
	journal = {Machine Learning},
	author = {Antos, András and Szepesvári, Csaba and Munos, Rémi},
	month = apr,
	year = {2008},
	pages = {89--129},
}

@book{altman_constrained_1999,
	address = {New York},
	title = {Constrained {Markov} {Decision} {Processes}: {Stochastic} {Modeling}},
	isbn = {978-1-315-14022-3},
	shorttitle = {Constrained {Markov} {Decision} {Processes}},
	publisher = {Routledge},
	author = {Altman, Eitan},
	month = dec,
	year = {1999},
	doi = {10.1201/9781315140223},
}

@inproceedings{zanette_bellman_2022,
	title = {Bellman {Residual} {Orthogonalization} for {Offline} {Reinforcement} {Learning}},
	url = {https://openreview.net/forum?id=x26Mpsf45P3},
	language = {en},
	urldate = {2023-06-05},
	author = {Zanette, Andrea and Wainwright, Martin J.},
	month = oct,
	year = {2022},
}

@misc{zanette_when_2022,
	title = {When is {Realizability} {Sufficient} for {Off}-{Policy} {Reinforcement} {Learning}?},
	url = {http://arxiv.org/abs/2211.05311},
	urldate = {2023-05-30},
	publisher = {arXiv},
	author = {Zanette, Andrea},
	month = nov,
	year = {2022},
	note = {arXiv:2211.05311 [cs]},
}

@article{wang_safe_2020,
	title = {Safe {Off}-{Policy} {Deep} {Reinforcement} {Learning} {Algorithm} for {Volt}-{VAR} {Control} in {Power} {Distribution} {Systems}},
	volume = {11},
	issn = {1949-3061},
	doi = {10.1109/TSG.2019.2962625},
	number = {4},
	journal = {IEEE Transactions on Smart Grid},
	author = {Wang, Wei and Yu, Nanpeng and Gao, Yuanqi and Shi, Jie},
	month = jul,
	year = {2020},
	note = {Conference Name: IEEE Transactions on Smart Grid},
	pages = {3008--3018},
}

@inproceedings{tang_model_2021,
	title = {Model {Selection} for {Offline} {Reinforcement} {Learning}: {Practical} {Considerations} for {Healthcare} {Settings}},
	shorttitle = {Model {Selection} for {Offline} {Reinforcement} {Learning}},
	url = {https://proceedings.mlr.press/v149/tang21a.html},
	language = {en},
	urldate = {2023-05-12},
	booktitle = {Proceedings of the 6th {Machine} {Learning} for {Healthcare} {Conference}},
	publisher = {PMLR},
	author = {Tang, Shengpu and Wiens, Jenna},
	month = oct,
	year = {2021},
	note = {ISSN: 2640-3498},
	pages = {2--35},
}

@inproceedings{xie_q_2020,
	title = {Q* {Approximation} {Schemes} for {Batch} {Reinforcement} {Learning}: {A} {Theoretical} {Comparison}},
	shorttitle = {Q* {Approximation} {Schemes} for {Batch} {Reinforcement} {Learning}},
	url = {https://proceedings.mlr.press/v124/xie20a.html},
	language = {en},
	urldate = {2023-05-01},
	booktitle = {Proceedings of the 36th {Conference} on {Uncertainty} in {Artificial} {Intelligence} ({UAI})},
	publisher = {PMLR},
	author = {Xie, Tengyang and Jiang, Nan},
	month = aug,
	year = {2020},
	note = {ISSN: 2640-3498},
	pages = {550--559},
}

@inproceedings{xie_batch_2021,
	title = {Batch {Value}-function {Approximation} with {Only} {Realizability}},
	url = {https://proceedings.mlr.press/v139/xie21d.html},
	language = {en},
	urldate = {2023-05-01},
	booktitle = {Proceedings of the 38th {International} {Conference} on {Machine} {Learning}},
	publisher = {PMLR},
	author = {Xie, Tengyang and Jiang, Nan},
	month = jul,
	year = {2021},
	note = {ISSN: 2640-3498},
	pages = {11404--11413},
}

@misc{zhu_importance_2023,
	title = {Importance {Weighted} {Actor}-{Critic} for {Optimal} {Conservative} {Offline} {Reinforcement} {Learning}},
	url = {http://arxiv.org/abs/2301.12714},
	doi = {10.48550/arXiv.2301.12714},
	urldate = {2023-02-12},
	publisher = {arXiv},
	author = {Zhu, Hanlin and Rashidinejad, Paria and Jiao, Jiantao},
	month = jan,
	year = {2023},
	note = {arXiv:2301.12714 [cs]},
}

@inproceedings{zhan_offline_2022,
	title = {Offline {Reinforcement} {Learning} with {Realizability} and {Single}-policy {Concentrability}},
	url = {https://proceedings.mlr.press/v178/zhan22a.html},
	language = {en},
	urldate = {2022-12-29},
	booktitle = {Proceedings of {Thirty} {Fifth} {Conference} on {Learning} {Theory}},
	publisher = {PMLR},
	author = {Zhan, Wenhao and Huang, Baihe and Huang, Audrey and Jiang, Nan and Lee, Jason},
	month = jun,
	year = {2022},
	note = {ISSN: 2640-3498},
	pages = {2730--2775},
}

@techreport{Gronauer2022BulletSafetyGym,
	author = {Gronauer, Sven},
	institution = {mediaTUM},
	title = {Bullet-Safety-Gym: A Framework for Constrained Reinforcement Learning},
	year = {2022},
	doi = {10.14459/2022md1639974}}

@article{chen2021decision,
  title={Decision transformer: Reinforcement learning via sequence modeling},
  author={Chen, Lili and Lu, Kevin and Rajeswaran, Aravind and Lee, Kimin and Grover, Aditya and Laskin, Misha and Abbeel, Pieter and Srinivas, Aravind and Mordatch, Igor},
  journal={Advances in neural information processing systems},
  volume={34},
  pages={15084--15097},
  year={2021}
}

@article{xie2021bellman,
  title={Bellman-consistent pessimism for offline reinforcement learning},
  author={Xie, Tengyang and Cheng, Ching-An and Jiang, Nan and Mineiro, Paul and Agarwal, Alekh},
  journal={Advances in neural information processing systems},
  volume={34},
  pages={6683--6694},
  year={2021}
}

@inproceedings{ozdaglar2023revisiting,
  title={Revisiting the linear-programming framework for offline rl with general function approximation},
  author={Ozdaglar, Asuman E and Pattathil, Sarath and Zhang, Jiawei and Zhang, Kaiqing},
  booktitle={International Conference on Machine Learning},
  pages={26769--26791},
  year={2023},
  organization={PMLR}
}

@article{liu2023datasets,
  title={Datasets and Benchmarks for Offline Safe Reinforcement Learning},
  author={Liu, Zuxin and Guo, Zijian and Lin, Haohong and Yao, Yihang and Zhu, Jiacheng and Cen, Zhepeng and Hu, Hanjiang and Yu, Wenhao and Zhang, Tingnan and Tan, Jie and others},
  journal={arXiv preprint arXiv:2306.09303},
  year={2023}
}

@article{dulacarnold2020realworldrlempirical,
  title={An empirical investigation of the challenges of real-world reinforcement learning},
  author={Dulac-Arnold, Gabriel and Levine, Nir and Mankowitz, Daniel J. and Li, Jerry and Paduraru, Cosmin and Gowal, Sven and Hester, Todd},
  year={2020},
}

@article{liu2023constrained,
  title={Constrained decision transformer for offline safe reinforcement learning},
  author={Liu, Zuxin and Guo, Zijian and Yao, Yihang and Cen, Zhepeng and Yu, Wenhao and Zhang, Tingnan and Zhao, Ding},
  journal={arXiv preprint arXiv:2302.07351},
  year={2023}
}

@inproceedings{xu2022constraints,
  title={Constraints penalized q-learning for safe offline reinforcement learning},
  author={Xu, Haoran and Zhan, Xianyuan and Zhu, Xiangyu},
  booktitle={Proceedings of the AAAI Conference on Artificial Intelligence},
  volume={36},
  number={8},
  pages={8753--8760},
  year={2022}
}

\newpage

% Use A, B, C, ... for the section numbers.
\renewcommand\thesection{\Alph{section}}
\setcounter{section}{0}

\title{Supplementary Materials}

\section{PERFORMANCE DIFFERENCE LEMMAS} \label{appendix:perf-diff-lemma}

In this section, we provide two generalizations of the classical performance difference lemma \parencite{kakade_approximately_2002}.
For completeness, we first state the classical performance difference lemma below.

\begin{lemma}[Performance Difference Lemma. \textcite{kakade_approximately_2002}] \label{lemma:perf-diff}
\begin{equation} \label{eqn:perf-diff}
(1 - \gamma)(J_U(\widehat\pi) - J_U(\pi)) = A_\pi(\widehat\pi, Q_U^{\widehat\pi})
\end{equation}
\end{lemma}

The following is the first generalization of the performance difference lemma.
It decomposes the difference in performance of two policies, where the performance of one of the policies is measured with respect to an arbitrary Q-value function $f$.
The same result is proved as an intermediate step in the proof of Lemma 12 in \textcite{cheng_adversarially_2022}.
We state it as a separate lemma and provide a simplified proof below.

\begin{lemma} \label{lemma:perf-diff-1}
For any functions $U, f: \mathcal{S} \times \mathcal{A} \rightarrow \mathbb{R}$ and any policies $\pi, \widehat{\pi} : \mathcal{S} \rightarrow \Delta(\mathcal{A})$, we have
$$
(1 - \gamma) (f(s_0, \widehat{\pi}) - J_U(\pi))
= A_\pi(\widehat{\pi}, f) + \mathbb{E}_\pi [ (f - \mathcal{T}_U^{\widehat{\pi}} f)(s, a)].
$$
\end{lemma}
\begin{proof}
Note that
\begin{align*}
\mathbb{E}_{(s, a) \sim d^\pi} [ f(s, \widehat\pi) ]
&= (1 - \gamma) \mathbb{E}^\pi \left[ \sum_{t = 0}^\infty \gamma^t f(s_t, \widehat\pi) \right] \\
&= (1 - \gamma) \mathbb{E}^\pi \left[ f(s_0, \widehat\pi) + \sum_{t = 0}^\infty \gamma^t \mathbb{E}^\pi \left[ \gamma f(s_{t + 1}, \widehat\pi) \mid s_t, a_t \right] \right] \\
&= (1 - \gamma) f(s_0, \widehat\pi) + (1 - \gamma) \mathbb{E}^\pi \left[ \sum_{t = 0}^\infty \gamma^t \mathcal{T}_0^{\widehat\pi} f(s_t, a_t) \right] \\
&= (1 - \gamma) f(s_0, \widehat\pi) + \mathbb{E}_{(s, a) \sim d^\pi} [ \mathcal{T}_0^{\widehat\pi} f(s, a)].
\end{align*}
Rearranging, we get
\begin{align*}
(1 - \gamma) f(s_0, \widehat\pi)
&=
\mathbb{E}_{(s, a) \sim d^\pi} [ (f - \mathcal{T}_0^{\widehat\pi} f)(s, a) ] - \mathbb{E}_{(s, a) \sim d^\pi} [f(s, a) - f(s, \widehat\pi)] \\
&=
\mathbb{E}_{(s, a) \sim d^\pi} [ (f - \mathcal{T}_0^{\widehat\pi} f)(s, a) ] + A_\pi(\widehat\pi, f).
\end{align*}
Using $(1 - \gamma) J_U(\pi) = \mathbb{E}_{(s, a) \sim d^\pi}[U(s, a)]$, we get
\begin{align*}
(1 - \gamma)(f(s_0, \widehat\pi) - J_U(\pi))
&= \mathbb{E}_\pi [(f - \mathcal{T}_0^{\widehat\pi} f)(s, a)] + A_\pi(\widehat\pi, f) - \mathbb{E}_\pi[U(s, a)] \\
&= A_\pi(\widehat\pi, f) + \mathbb{E}_\pi[(f - \mathcal{T}_U^{\widehat\pi} f)(s, a)].
\end{align*}
\end{proof}

Note that when we set $f = Q_U^{\widehat\pi}$ in the lemma above, we recover the classical performance difference lemma.
Now, we state the second generalization of the performance difference lemma.
The same lemma is stated and proved in \textcite{cheng_adversarially_2022} and also used in \textcite{zhu_importance_2023}.
We state the lemma and provide a simpler proof below.

\begin{lemma}[Performance difference decomposition. Lemma 12 in \textcite{cheng_adversarially_2022}] \label{lemma:perf-diff-2}
For any policies $\pi, \widehat\pi, \mu : \mathcal{S} \rightarrow \Delta(\mathcal{A})$ and any functions $U: \mathcal{S} \times \mathcal{A} \rightarrow \mathbb{R}$ and $f: \mathcal{S} \times \mathcal{A} \rightarrow \mathbb{R}$, we have
$$
\begin{aligned}
(1 - \gamma)(J_U(\pi) - J_U(\widehat\pi))
&= \mathbb{E}_{\mu}[(f - \mathcal{T}^{\widehat\pi}_U f)(s, a)]
+ \mathbb{E}_{\pi}[(\mathcal{T}^{\widehat\pi}_U f - f)(s, a)] \\
&\quad\quad\quad+ \mathbb{E}_{\pi}[f(s, \pi) - f(s, \widehat\pi)]
+ A_{\mu}(\widehat\pi, f) - A_{\mu}(\widehat\pi, Q_U^{\widehat\pi}).
\end{aligned}
$$
\end{lemma}

\begin{proof}
We have
\begin{align*}
(1 &- \gamma)(J_U(\pi) - J_U(\widehat\pi)) \\
&=
(1 - \gamma)(J_U(\pi) - f(s_0, \widehat\pi)) + (1 - \gamma)(f(s_0, \widehat\pi) - J_U(\mu))
+ (1 - \gamma)(J_U(\mu) - J_U(\widehat\pi)) \\
&=
- (A_{\pi}(\widehat\pi, f) + \mathbb{E}_{\pi} [(f - \mathcal{T}_U^{\widehat\pi} f)(s, a)])
+ (A_{\mu}(\widehat\pi, f) + \mathbb{E}_{\mu} [(f - \mathcal{T}_U^{\widehat\pi} f)(s, a)])
- A_{\mu}(\widehat\pi, Q_U^{\pi})
\end{align*}
where the second inequality uses the generalized performance difference lemma (Lemma~\ref{lemma:perf-diff-1}) for the first two terms and the classical performance difference lemma (Lemma~\ref{lemma:perf-diff}) for the third term.
Rearranging and observing that $A_{\pi}(\widehat\pi, f) = \mathbb{E}_{\pi}[f(s, \widehat\pi) - f(s, \pi)]$ complete the proof.
\end{proof}

Indeed, the lemma above is a generalization because setting $\mu = \pi$ reduces to the classical performance difference lemma (Lemma~\ref{lemma:perf-diff}).

\section{CONCENTRATION INEQUALITIES} \label{appendix:concentration}

In this section, we provide concentration inequalities for relating $\mathcal{E}_\mu$ and $A_\mu$ to the empirical versions $\mathcal{E}_\mathcal{D}$ and $A_\mathcal{D}$ respectively.
First, we show a concentration bound on $\mathbb{E}_\mathcal{D}[w(s, a) (f(s, a) - U(s, a) - \gamma f(s', \pi))]$, which will be used to show a concentration bound on $\mathcal{E}_\mathcal{D}(\pi, f, U)$.

\begin{lemma}[Concentration of Bellman error] \label{lemma:concentration-bellman-error}
Let $w : \mathcal{S} \times \mathcal{A} \rightarrow \mathbb{R}_+$ with $\Vert w \Vert \leq C_\infty$ and $\Vert w \Vert_{2, \mu} \leq C_{\ell_2}$.
Let $f : \mathcal{S} \times \mathcal{A} \rightarrow [0, \frac{1}{1 - \gamma}]$ and $U : \mathcal{S} \times \mathcal{A} \rightarrow [0, 1]$ be any functions.
Let $\pi : \mathcal{S} \rightarrow \Delta(\mathcal{A})$ be any policy.
With probability at least $1 - \delta$, we have
$$
\begin{aligned}
\vert
\mathbb{E}_\mu[w(s, a) (f - \mathcal{T}_U^\pi)(s, a)]
-
\mathbb{E}_\mathcal{D}&[w(s, a) (f(s, a) - U(s, a) - \gamma f(s', \pi))]
\vert \\
&\leq \mathcal{O} \left(
\frac{C_{\ell_2}}{1 - \gamma} \sqrt{
\frac{\log(1 / \delta)}{n}
}
+
\frac{C_\infty}{1 - \gamma} \frac{\log(1 / \delta)}{n}
\right).
\end{aligned}
$$
\end{lemma}
\begin{proof}
Define
$$
X_j = w(s_j, a_j) (f(s_j, a_j) - U(s_j, a_j) - \gamma f(s_j', \pi)).
$$
Note that $\mathbb{E}_\mathcal{D}[w(s, a)(f(s, a) - U(s, a) - \gamma f(s', \pi))] = \frac{1}{n} \sum_{j = 1}^n X_j$.
By assumption of the data distribution, $X_1, \dots, X_n$ are i.i.d.
By the boundedness assumption on $w$ and $f$, we have $\vert X_j \vert \leq \mathcal{O}(C_\infty / (1 - \gamma))$.
Also, we have
\begin{align*}
\mathbb{E} [X_j]
&= \mathbb{E}_{(s, a) \sim d^\mu, s' \sim P(\cdot | s, a)} [w(s, a) (f(s, a) - U(s, a) - \gamma f(s', \pi))] \\
&= \mathbb{E}_{(s, a) \sim d^\mu} [\mathbb{E}_{s' \sim P(\cdot | s, a)} [w(s, a)(f(s, a) - U(s, a) - \gamma f(s', \pi)) | s, a]] \\
&= \mathbb{E}_\mu [ w(s, a)(f - \mathcal{T}_U^\pi f)(s, a)].
\end{align*}
By Bernstein's inequality, we have with probability at least $1 - \delta$ that
\begin{align*}
\Bigg\vert
\mathbb{E}_\mu [w(s, a)(f &- \mathcal{T}_U^\pi f)(s, a)] - \frac{1}{n} \sum_{j = 1}^n X_j
\Bigg\vert \\
&\leq
\mathcal{O}\left(
\sqrt{\frac{\text{Var}_\mu[w(s, a)(f - \mathcal{T}_U^\pi f)(s, a)] \log (1 / \delta)}{n}} + \frac{C_\infty \log(1 / \delta)}{(1 - \gamma)n}
\right)
\end{align*}
The variance term $\text{Var}_\mu[w(s, a)(f - \mathcal{T}_U^\pi f)(s, a)]$ can be bounded as follows.
\begin{align*}
\text{Var}_\mu[w(s, a)(f - \mathcal{T}_U^\pi f)(s, a)]
&\leq
\mathbb{E}_\mu[w(s, a)^2 (f - \mathcal{T}_U^\pi f)^2(s, a)] \\
&\leq
\mathcal{O}(\Vert w \Vert_{2, \mu}^2 / (1 - \gamma)^2) \\
&\leq
\mathcal{O}((C_{\ell_2})^2 / (1 - \gamma)^2)
\end{align*}
where the second inequality uses the boundedness assumption on $f$ and the last inequality uses the boundedness assumption on $w$.
Hence, we have
$$
\Bigg\vert
\mathbb{E}_\mu [w(s, a)(f - \mathcal{T}_U^\pi f)(s, a)] - \frac{1}{n} \sum_{j = 1}^n X_j
\Bigg\vert
\leq
\mathcal{O}\left(
\frac{C_{\ell_2}}{1 - \gamma} \sqrt{\frac{\log (1 / \delta)}{n}} + \frac{C_\infty \log(1 / \delta)}{(1 - \gamma)n}
\right).
$$
This completes the proof.
\end{proof}

The following lemma relates $\mathcal{E}_\mu$ to $\mathcal{E}_\mathcal{D}$.
The proof closely follows that of Lemma~4 in \textcite{zhu_importance_2023}, which shows the same result for a single reward function.

\begin{lemma}[Concentration of Bellman error term] \label{lemma:concentration}
Under Assumption~\ref{assumption:w-boundedness}, with probability at least $1 - \delta$, we have
\begin{align*}
\vert \mathcal{E}_\mu(\pi, f; R) - \mathcal{E}_\mathcal{D}(\pi, f; R) \vert &\leq \mathcal{O}(\epsilon_{\text{stat}}) \\
\vert \mathcal{E}_\mu(\pi, g_i; C_i) - \mathcal{E}_\mathcal{D}(\pi, g_i; C_i) \vert &\leq \mathcal{O}(\epsilon_{\text{stat}}) ~~\text{for all}~i = 1, \dots, I,
\end{align*}
for any $\pi \in \Pi$, $f \in \mathcal{F}$ and $g_i \in \mathcal{F}$, $i = 1, \dots, I$, where
$$
\epsilon_{\text{stat}} \coloneqq \frac{C_{\ell_2}}{1 - \gamma} \sqrt{\frac{\log (I \vert \mathcal{F} \vert \vert \Pi \vert \vert \mathcal{W} \vert / \delta)}{n}} + \frac{C_\infty \log(I \vert \mathcal{F} \vert \vert \Pi \vert \vert \mathcal{W} \vert / \delta)}{(1 - \gamma)n}.
$$
\end{lemma}

\begin{proof}
It is enough to show that, with probability $1 - \delta$, we have
$\vert \mathcal{E}_\mu(\pi, h; U) - \mathcal{E}_\mathcal{D}(\pi, h; U) \vert \leq \mathcal{O}(\epsilon_{\text{stat}})$
for any policy $\pi \in \Pi$, any function $h \in \mathcal{F}$ and any function $U \in \{R, C_1, \dots, C_I\}$.
Fix $\pi \in \Pi$, $h \in \mathcal{F}$ and $U \in \{R, C_1, \dots, C_I\}$, and define
$$
X_j(w) = w(s_j, a_j) (h(s_j, a_j) - U(s_j, a_j) - \gamma h(s_j', \pi))
$$
where $w \in \mathcal{W}$.
By Lemma~\ref{lemma:concentration-bellman-error}, we have
$$
\Bigg\vert
\mathbb{E}_\mu [w(s, a)(h - \mathcal{T}_U^\pi h)(s, a)] - \frac{1}{n} \sum_{j = 1}^n X_j(w)
\Bigg\vert
\leq
\mathcal{O}\left(
\frac{C_{\ell_2}}{1 - \gamma} \sqrt{\frac{\log (1 / \delta)}{n}} + \frac{C_\infty \log(1 / \delta)}{(1 - \gamma)n}
\right).
$$
Since the inequality above holds for all $w \in \mathcal{W}$, it follows by a union bound over $w \in \mathcal{W}$ that
\begin{align*}
\mathcal{E}_\mu&(\pi, h; U) - \mathcal{E}_\mathcal{D}(\pi, h; U)
=
\max_{w \in \mathcal{W}} \left\vert \mathbb{E}_\mu \left[
w(s, a) (h - \mathcal{T}_U^\pi h)(s, a)
\right] \right\vert \\
&\quad\quad-
\max_{w \in \mathcal{W}} \left\vert \mathbb{E}_\mathcal{D} \left[
w(s, a)(h(s, a) - U(s, a) - \gamma h(s', \pi))
\right] \right\vert \\
&\leq
\left\vert \mathbb{E}_\mu \left[
w^\star(s, a) (h - \mathcal{T}_U^\pi h)(s, a)
\right] \right\vert
-
\left\vert \mathbb{E}_\mathcal{D} \left[
w^\star(s, a)(h(s, a) - U(s, a) - \gamma h(s', \pi))
\right] \right\vert \\
&\leq
\left\vert \mathbb{E}_\mu \left[
w^\star(s, a) (h - \mathcal{T}_U^\pi h)(s, a)
\right] -
\mathbb{E}_\mathcal{D} \left[
w^\star(s, a)(h(s, a) - U(s, a) - \gamma h(s', \pi))
\right] \right\vert \\
&=
\Bigg\vert
\mathbb{E}_\mu [w^\star(s, a)(h - \mathcal{T}_U^\pi h)(s, a)] - \frac{1}{n} \sum_{j = 1}^n X_j(w^\star)
\Bigg\vert \\
&\leq
\mathcal{O}\left(
\frac{C_{\ell_2}}{1 - \gamma} \sqrt{\frac{\log (\vert \mathcal{W} \vert / \delta)}{n}} + \frac{C_\infty \log(\vert \mathcal{W} \vert / \delta)}{(1 - \gamma)n}
\right)
\end{align*}
where in the first inequality, we use the notation $w^\star = \argmax_{w \in \mathcal{W}} \vert \mathbb{E}_\mu [w(s, a)(f - \mathcal{T}_U^\pi f)(s, a)] \vert$;
the second inequality follows by the identity $\vert a \vert - \vert b \vert \leq \vert a - b \vert$;
and the last inequality uses the previous result.
The bound of $\mathcal{E}_\mathcal{D}(\pi, h; U) - \mathcal{E}_\mu(\pi, h; U)$ follows similarly, and the union bound of the two bounds gives
$$
\vert \mathcal{E}_\mu(\pi, h; U) - \mathcal{E}_\mathcal{D}(\pi, h; U) \vert
\leq
\mathcal{O}\left(
\frac{C_{\ell_2}}{1 - \gamma} \sqrt{\frac{\log (\vert \mathcal{W} \vert / \delta)}{n}} + \frac{C_\infty \log(\vert \mathcal{W} \vert / \delta)}{(1 - \gamma)n}
\right).
$$
A union bound on all $(h, \pi, U) \in \mathcal{F} \times \Pi \times \{R, C_1, \dots, C_I\}$ completes the proof.
\end{proof}

The following lemma relates $A_\mu$ to $A_\mathcal{D}$.
The proof closely follows that of Lemma~5 in \textcite{zhu_importance_2023}, which shows the same result for a single reward function.

\begin{lemma}[Concentration of the advantage function] \label{lemma:concentration-advantage}
With probability at least $1 - \delta$, for any $\pi \in \Pi$, $f \in \mathcal{F}$, we have
$$
\vert A_\mu(\pi, f) - A_\mathcal{D}(\pi, f) \vert
\leq
\mathcal{O}\left(
\sqrt{\frac{\log (\vert \mathcal{F} \vert \vert \Pi \vert / \delta)}{n(1 - \gamma)^2}}
\right)
\leq
\mathcal{O}(\epsilon_{\text{stat}})
$$
where $\epsilon_{\text{stat}}$ is defined in Lemma~\ref{lemma:concentration}.
\end{lemma}

\begin{proof}
Note that $\mathbb{E}[A_\mathcal{D}(\pi, f)] = A_\mu(\pi, f)$ and $\vert f(s_i, \pi) - f(s_i, a_i) \vert \leq \mathcal{O}(\frac{1}{1 - \gamma})$.
Fixing $\pi \in \mathcal{F}$ and $f \in \mathcal{F}$ and applying Hoeffding's inequality, we have with probability at least $1 - \delta$ that
$$
\vert A_\mu(\pi, f) - A_\mathcal{D}(\pi, f) \vert
\leq
\mathcal{O}\left(
\sqrt{\frac{\log (1 / \delta)}{n(1 - \gamma)^2}}
\right).
$$
Applying union bound on $(\pi, f) \in \Pi \times \mathcal{F}$ completes the proof.
\end{proof}

\section{PDCA PRODUCES A NEAR SADDLE POINT} \label{appendix:saddle-point}

In this section, we show that our algorithm PDCA (Algorithm~\ref{alg:pdapc}) produces a near saddle point.

\begin{lemma} \label{lemma:saddle-point}
Consider the policy $\bar\pi$ produced by the algorithm PDCA (Algorithm~\ref{alg:pdapc}) with threshold $\bm\tau$ and bound $B$.
Let $\bar{\bm\lambda} \coloneqq \frac{1}{K} \sum_{k = 1}^K \bm\lambda_k$ where $\bm\lambda_1, \dots, \bm\lambda_K \in B \bm\Delta^I$ is the sequence of Lagrange multipliers produced by the $\lambda$-player in the algorithm.
Then, under Assumption~\ref{assumption:value-realizability},\ref{assumption:concentrability},\ref{assumption:w-boundedness} and \ref{assumption:miw-realizability}, with probability at least $1 - \delta$, we have for any $\pi \in \text{Conv}(\Pi)$ satisfying $w^\pi \in \mathcal{W}$ and any $\bm\lambda \in B \bm\Delta^I$ that
$$
L(\pi, \bar{\bm\lambda}) \leq L(\bar\pi, \bm\lambda) + \frac{1 + 2B}{(1 - \gamma)^2} \epsilon_{\text{opt}}(K) + \mathcal{O}\left(
\frac{B \epsilon_{\text{stat}}}{1 - \gamma}
\right).
$$
\end{lemma}

We prove Lemma~\ref{lemma:saddle-point} by bounding the regrets of the $\pi$-player and $\lambda$-player against their best actions in hindsight.

\subsection{Bounding the Regret of the $\pi$-Player} \label{appendix:pi-regret}

\begin{lemma} \label{lemma:saddle-point-step-1}
Under Assumption~\ref{assumption:value-realizability},\ref{assumption:w-boundedness} and \ref{assumption:miw-realizability}, with probability at least $1 - \delta$, the sequences $\pi_1, \dots, \pi_K \in \Pi$ and $\bm\lambda_1, \dots, \bm\lambda_K \in B \bm\Delta^I$ produced by the algorithm $\textsc{PDCA}$ using the cost threshold $\bm\tau$ and the bound $B$ satisfy
$$
L(\pi, \bar{\bm\lambda}) - \frac{1}{K} \sum_{k = 1}^K L(\pi_k, \bm\lambda_k)
\leq
\frac{1 + 2B}{(1 - \gamma)^2} \epsilon_{\text{opt}}(K) + \mathcal{O}\left(\frac{B \epsilon_{\text{stat}}}{1 - \gamma} \right)
$$
for all $\pi \in \text{Conv}(\Pi)$ with $w^\pi \in \mathcal{W}$ where $L(\pi, \bm\lambda) = J_R(\pi) + \bm\lambda \cdot (\bm\tau - J_{\bm{C}}(\pi))$.
\end{lemma}
\begin{proof}
Fix a policy $\pi \in \text{Conv}(\Pi)$ satisfying $w^\pi \in \mathcal{W}$.
Such a policy exists by Assumption~\ref{assumption:miw-realizability}.
By the definition of $L(\cdot, \cdot)$, we get
$$
L(\pi, \bar{\bm\lambda}) - \frac{1}{K} \sum_{k = 1}^K L(\pi_k, \bm\lambda_k)
=
\frac{1}{K} \sum_{k = 1}^K (\underbrace{J_R(\pi) - J_R(\pi_k)}_{(a)}) + \frac{1}{K} \sum_{k = 1}^K \sum_{i = 1}^I \lambda_k^i(\underbrace{J_{C_i}(\pi_k) - J_{C_i}(\pi)}_{(b)}).
$$

By a union bound with $\delta$ scaled appropriately, the concentration bounds for $\mathcal{E}_\mathcal{D}$ and $A_\mathcal{D}$ in Lemma~\ref{lemma:concentration} and Lemma~\ref{lemma:concentration-advantage}, and the regret bound of the oracle used by the $\pi$-player (Definition~\ref{def:pi-player}) hold with probability at least $1 - \delta$.
For the rest of the proof, we assume these events hold.

\paragraph{Bounding $(a)$}
We use the performance difference lemma (Lemma~\ref{lemma:perf-diff-2}) to bound $(a)$ as follows.
\begin{align*}
(1 - \gamma) &(J_R(\pi) - J_R(\pi_k)) \\
&=
\mathbb{E}_\mu[(f_k - \mathcal{T}_R^{\pi_k} f_k)(s, a)]
+ \mathbb{E}_\pi[(\mathcal{T}_R^{\pi_k} f_k - f_k)(s, a)] \\
&\quad\quad\quad+ \mathbb{E}_\pi[f_k(s, \pi) - f_k(s, \pi_k)]
+ A_\mu(\pi_k, f_k) - A_\mu(\pi_k, Q_R^{\pi_k}) \\
&\leq
2 \mathcal{E}_\mu(\pi_k, f_k; R)
+ \mathbb{E}_\pi[f_k(s, \pi) - f_k(s, \pi_k)]
+ A_\mu(\pi_k, f_k) - A_\mu(\pi_k, Q_R^{\pi_k}) \\
&\leq
\mathbb{E}_\pi[f_k(s, \pi) - f_k(s, \pi_k)]
+ \underbrace{2 \mathcal{E}_\mathcal{D}(\pi_k, f_k; R)
+ A_\mathcal{D}(\pi_k, f_k)}_{(\star)} - A_\mathcal{D}(\pi_k, Q_R^{\pi_k})
+ \mathcal{O}(\epsilon_{\text{stat}})
\end{align*}
where the first inequality follows by $w^\pi \in \mathcal{W}$ which implies $ \vert \mathbb{E}_\pi[(f - \mathcal{T}_U^\pi f)(s, a)] \vert \leq \mathcal{E}_\mu(\pi, f; U)$;
and the second inequality follows by the concentration results in Lemma~\ref{lemma:concentration} and Lemma~\ref{lemma:concentration-advantage}.
Recall that the reward critic chooses $f_k \in \mathcal{F}$ that minimizes $2 \mathcal{E}_\mathcal{D} (\pi_k, \cdot; R) + A_\mathcal{D}(\pi_k, \cdot)$.
We have $Q_R^{\pi_k} \in \mathcal{F}$ by the realizability assumption (Assumption~\ref{assumption:value-realizability}).
Hence,
$$
2 \mathcal{E}_\mathcal{D}(\pi_k, f_k; R) + A_\mathcal{D}(\pi_k, f_k)
\leq
2 \mathcal{E}_\mathcal{D}(\pi_k, Q_R^{\pi_k}; R) + A_\mathcal{D}(\pi_k, Q_R^{\pi_k}).
$$
Using this inequality for bounding $(\star)$ and continuing the bound of $J_R(\pi) - J_R(\pi_k)$, we get
\begin{align*}
(1 - \gamma) (J_R(\pi) - J_R(\pi_k))
&\leq
\mathbb{E}_\pi[f_k(s, \pi) - f_k(s, \pi_k)]
+ 2 \mathcal{E}_\mathcal{D}(\pi_k, Q_R^{\pi_k}; R)
+ \mathcal{O}(\epsilon_{\text{stat}}) \\
&\leq
\mathbb{E}_\pi[f_k(s, \pi) - f_k(s, \pi_k)]
+ 2 \mathcal{E}_\mu(\pi_k, Q_R^{\pi_k}; R) + \mathcal{O}(\epsilon_{\text{stat}}) \\
&=
\mathbb{E}_\pi[f_k(s, \pi) - f_k(s, \pi_k)] + \mathcal{O}(\epsilon_{\text{stat}})
\end{align*}
where the last equality uses the fact that $Q_R^\pi$ solves $f - \mathcal{T}_R^\pi f = 0$, which gives $\mathcal{E}_\mu(\pi_k, Q_R^{\pi_k}; R) = \max_{w \in \mathcal{W}} \vert \mathbb{E}_\mu[ w(s, a) (Q_R^{\pi_k} - \mathcal{T}_R^{\pi_k} Q_R^{\pi_k})] \vert = 0$.

\paragraph{Bounding $(b)$}
Similarly, we can bound $(b)$ as follows.
\begin{align*}
(1 - \gamma)&(J_{C_i}(\pi_k) - J_{C_i}(\pi)) \\
&=
\mathbb{E}_\mu[(\mathcal{T}_{C_i}^{\pi_k} g_k^i - g_k^i)(s, a)]
+ \mathbb{E}_\pi[(g_k^i - \mathcal{T}_{C_i}^{\pi_k} g_k^i)(s, a)] \\
&\quad\quad\quad+ \mathbb{E}_\pi[g_k^i(s, \pi_k) - g_k^i(s, \pi)]
- A_\mu(\pi_k, g_k^i) + A_\mu(\pi_k, Q_{C_i}^{\pi_k}) \\
&\leq
2 \mathcal{E}_\mu(\pi_k, g_k^i; C_i)
+ \mathbb{E}_\pi[g_k^i(s, \pi_k) - g_k^i(s, \pi)]
- A_\mu(\pi_k, g_k^i) + A_\mu(\pi_k, Q_{C_i}^{\pi_k}) \\
&\leq
\mathbb{E}_\pi[g_k^i(s, \pi_k) - g_k^i(s, \pi)]
+ 2 \mathcal{E}_\mathcal{D}(\pi_k, g_k^i; C_i)
- A_\mathcal{D}(\pi_k, g_k^i) + A_\mathcal{D}(\pi_k, Q_{C_i}^{\pi_k}) + \mathcal{O}(\epsilon_{\text{stat}}) \\
&\leq
\mathbb{E}_\pi[g_k^i(s, \pi_k) - g_k^i(s, \pi)]
+ 2 \mathcal{E}_\mathcal{D}(\pi_k, Q_{C_i}^{\pi_k}; C_i)
+ \mathcal{O}(\epsilon_{\text{stat}}) \\
&\leq
\mathbb{E}_\pi[g_k^i(s, \pi_k) - g_k^i(s, \pi)]
+ 2 \mathcal{E}_\mu(\pi_k, Q_{C_i}^{\pi_k}; C_i) + \mathcal{O}(\epsilon_{\text{stat}}) \\
&=
\mathbb{E}_\pi[g_k^i(s, \pi_k) - g_k^i(s, \pi)] + \mathcal{O}(\epsilon_{\text{stat}})
\end{align*}
where the third inequality uses the realizability assumption (Assumption~\ref{assumption:value-realizability}) for $Q_{C_i}^{\pi_k} \in \mathcal{F}$ and the fact that the cost critic chooses $g_k^i \in \mathcal{F}$ that minimizes $2\mathcal{E}_\mathcal{D}(\pi_k, \cdot; C_i) - A_\mathcal{D}(\pi_k, \cdot)$.

\paragraph{Using the Property of $\pi$-Player}
Using the bounds for $(a)$ and $(b)$ and continuing, we get
\begin{align*}
\frac{1 - \gamma}{K} \sum_{k = 1}^K &(L(\pi, \bm\lambda_k) - L(\pi_k, \bm\lambda_k)) \\
&=
\frac{1 - \gamma}{K} \sum_{k = 1}^K (J_R(\pi) - J_R(\pi_k)) + \frac{1 - \gamma}{K} \sum_{k = 1}^K \sum_{i = 1}^I \lambda_k^i(J_{C_i}(\pi_k) - J_{C_i}(\pi)) \\
&\leq
\frac{1}{K} \sum_{k = 1}^K \Bigg(
\mathbb{E}_\pi [ f_k(s, \pi) - f_k(s, \pi_k) ] + \sum_{i = 1}^I \lambda_k^i \mathbb{E}_\pi [ g_k^i(s, \pi_k) - g_k^i(s, \pi) ]
\Bigg) + \mathcal{O} \left(B \epsilon_{\text{stat}} \right) \\
&=
\frac{1}{K} \sum_{k = 1}^K \mathbb{E}_\pi [z_k(s, \pi) - z_k(s, \pi_k)] + \mathcal{O} \left(B \epsilon_{\text{stat}} \right) \\
&\leq
\frac{1 + 2B}{1 - \gamma} \epsilon_{\text{opt}}(K) + \mathcal{O} \left(B \epsilon_{\text{stat}} \right)
\end{align*}
where $z_k = f_k + \sum_{i = 1}^I \lambda_k^i(\tau_i - g_k^i)$ and the last inequality follows by the property of the policy optimization oracle (Definition~\ref{def:pi-player}) employed by the $\pi$-player and the fact that $\vert z_k(s, a) \vert \leq \frac{1 + 2B}{1 - \gamma}$ for all $s \in \mathcal{S}$ and $a \in \mathcal{A}$.
Rearranging completes the proof.
\end{proof}

\subsection{Bound the Regret of the $\lambda$-Player}

\begin{lemma} \label{lemma:saddle-point-step-2}
Under Assumption~\ref{assumption:value-realizability},\ref{assumption:concentrability},\ref{assumption:w-boundedness} and \ref{assumption:miw-realizability}, with probability at least $1 - \delta$, the sequences $\pi_1, \dots, \pi_K \in \Pi$ and $\bm\lambda_1, \dots, \bm\lambda_K \in B \bm\Delta^I$ produced by the algorithm $\textsc{PDCA}$ using the cost threshold $\bm\tau$ and the bound $B$ satisfy
$$
\frac{1}{K} \sum_{k = 1}^K L(\pi_k, \bm{\lambda}_k)
\leq
\frac{1}{K} \sum_{k = 1}^K L(\pi_k, \bm\lambda) + \mathcal{O}\left(
\frac{B \epsilon_{\text{stat}}}{1 - \gamma}
\right)
$$
for all $\bm\lambda \in B \bm\Delta^I$ where $L(\pi, \bm\lambda) = J_R(\pi) + \bm\lambda \cdot (\bm\tau - J_{\bm{C}}(\pi))$.
\end{lemma}
\begin{proof}
Recall that the OPE oracle produces an estimate $h$ for the value of $\pi$ with respect to a utility function $U$ that satisfies $\vert J_U(\pi) - h \vert \leq \mathcal{O}(\frac{C_{\ell_2}}{1 - \gamma} \sqrt{\frac{\log(\vert \mathcal{F} \vert / \delta)}{n}}) \leq \mathcal{O}(\epsilon_{\text{stat}})$ with probability at least $1 - \delta$.
By applying a union bound on $(\pi, U) \in \Pi \times \{C_1, \dots, C_I\}$, we have with probability at least $1 - \delta$ that
$$
(1 - \gamma)\vert J_{C_i}(\pi_k) - h_k^i \vert \leq \mathcal{O}(\epsilon_{\text{stat}})
$$
for all $k = 1, \dots, K$ and all $i = 1, \dots, I$.
Hence,
\begin{align*}
\frac{1}{K} \sum_{k = 1}^K L(\pi_k, \bm\lambda_k) - \frac{1}{K} \sum_{k = 1}^K L(\pi_k, \bm\lambda) &=
\frac{1}{K} \sum_{k = 1}^K \sum_{i = 1}^I (\lambda_k^i - \lambda^i)(\tau_i - J_{C_i}(\pi_k)) \\
&\leq
\frac{1}{K} \sum_{k = 1}^K \sum_{i = 1}^I (\lambda_k^i - \lambda^i)(\tau_i - h_k^i)
+ \mathcal{O} \left(
\frac{B \epsilon_{\text{stat}}}{1 - \gamma}
\right).
\end{align*}
The first term in the last expression is $\frac{1}{K} \sum_{k = 1}^K \bm\lambda_k \cdot (\bm\tau - \bm{h}_k) - \frac{1}{K} \sum_{k = 1}^K \bm\lambda \cdot (\bm\tau - \bm{h}_k) \leq 0$ since the $\lambda$-player chooses $\bm\lambda_k$ greedily that minimizes $\bm\lambda \mapsto \bm\lambda \cdot(\bm\tau - \bm{h}_k)$, and we are done.
\end{proof}

\subsection{Proof of Lemma~\ref{lemma:saddle-point}}

Combining the results of Lemma~\ref{lemma:saddle-point-step-1} and Lemma~\ref{lemma:saddle-point-step-2}, we can show that the pair $(\bar\pi, \bar{\bm\lambda})$ is approximately a saddle point where $\bar\pi$ is the policy returned by PDCA and $\bar{\bm\lambda}$ is the average of the sequence of Lagrange multipliers $\bm\lambda_1, \dots, \bm\lambda_K$ produced by PDCA.

\begin{refproof}[Lemma~\ref{lemma:saddle-point}]
Fix a policy $\pi \in \text{Conv}(\Pi)$ and a Lagrange multiplier $\bm\lambda \in B \bm\Delta^I$.
Then,
\begin{align*}
L(\pi, \bar{\bm\lambda})
&\leq
\frac{1}{K} \sum_{k = 1}^K L(\pi_k, \bm\lambda_k) +
\frac{1 + 2B}{(1 - \gamma)^2} \epsilon_{\text{opt}} (K) + \mathcal{O}\left(\frac{B \epsilon_{\text{stat}}}{1 - \gamma} \right) \\
&\leq
\frac{1}{K} \sum_{k = 1}^K  L(\pi_k, \bm\lambda) + \frac{1 + 2B}{(1 - \gamma)^2} \epsilon_{\text{opt}}(K) + \mathcal{O}\left(
\frac{B \epsilon_{\text{stat}}}{1 - \gamma}
\right)
\end{align*}
where the first inequality uses Lemma~\ref{lemma:saddle-point-step-1} and the second inequality uses Lemma~\ref{lemma:saddle-point-step-2}.
Observing that $\frac{1}{K}\sum_{k = 1}^K L(\pi_k, \bm\lambda) = L(\bar\pi, \bm\lambda)$ by the linearity of $L(\cdot, \bm\lambda)$ completes the proof.
\end{refproof}

\section{PROPERTIES OF A NEAR SADDLE POINT} \label{appendix:saddle-point-property}

In this section, we study the properties of a near saddle point formally defined below.

\begin{definition} \label{def:near-saddle-point}
We say $(\bar{x}, \bar{y})$ is a $\xi$-near saddle point for a function $L(\cdot, \cdot)$ with respect to the input space $\mathcal{X} \times \mathcal{Y}$ if
$L(x, \bar{y}) \leq L(\bar{x}, y) + \xi$ for all $x \in \mathcal{X}$ and $y \in \mathcal{Y}$.
\end{definition}

\begin{lemma} \label{lemma:key1}
Suppose $(\bar\pi, \bar{\bm\lambda})$ is a $\xi$-near saddle point for $L(\cdot, \cdot)$ with respect to $\widetilde\Pi \times B \bm\Delta^I$ where $\widetilde\Pi \subseteq \text{Conv}(\Pi)$ is a class of mixtures of policies and at least one mixture policy in $\widetilde\Pi$ is feasible for $\mathcal{P}(\bm\tau)$.
Then, we have
\begin{align}
J_R(\bar\pi) &\geq J_R(\pi_c) - \xi \tag{Optimality} \\
J_{C_i}(\bar\pi) &\leq \tau_i + \frac{\xi}{B} + \frac{1}{B(1 - \gamma)}, \quad \text{for all}~i = 1, \dots I \tag{Feasibility}
\end{align}
where $\pi_c$ is any feasible policy in $\widetilde\Pi$.
\end{lemma}
\begin{proof}
We first prove $J_R(\bar\pi) \geq J_R(\pi_c) - \xi$, near optimality of $\bar\pi$.

\paragraph{Optimality}

Since $(\bar\pi, \bar{\bm\lambda})$ is a $\xi$-near saddle point for $L(\cdot, \cdot)$ with respect to $\widetilde\Pi \times B \bm\Delta^I$ and $\pi_c \in \widetilde\Pi$, we have $L(\pi_c, \bar{\bm\lambda}) \leq L(\bar\pi, \bm\lambda) + \xi$ for all $\bm\lambda \in B \bm\Delta^I$.
Choosing $\bm\lambda = \bm{0}$, we get
$$
L(\pi_c, \bar{\bm\lambda}) \leq L(\bar\pi, 0) + \xi = J_R(\bar\pi) + \xi.
$$
Rearranging, we get
\begin{align*}
J_R(\bar\pi)
&\geq
J_R(\pi_c) + \bar{\bm\lambda} \cdot (\bm\tau - J_{\bm{C}}(\pi_c)) - \xi \\
&\geq
J_R(\pi_c) - \xi
\end{align*}
where the second inequality uses the feasibility of $\pi_c$ for $\mathcal{P}(\bm\tau)$.
This proves the near optimality of $\bar\pi$ with respect to $\pi_c$.

\paragraph{Feasibility}

Now, to prove near feasibility of $\bar\pi$, recall from the proof of the near optimality that $L(\pi_c, \bar{\bm\lambda}) \leq L(\bar\pi, \bm\lambda) + \xi$ for all $\bm\lambda \in B \bm\Delta^I$ holds since $(\bar\pi, \bar{\bm\lambda})$ is a $\xi$-near saddle point.
Choosing $\bm\lambda$ such that $\lambda_j = B$ for $j = \argmin_{i \in [I]} (\tau_i - J_{C_i}(\bar\pi))$ and $\lambda_j = 0$ for other $j$'s, and defining $m = \min_{i \in [I]}(\tau_i - J_{C_i}(\bar\pi))$, we get
$$
L(\pi_c, \bar{\bm\lambda}) \leq L(\bar\pi, \bm\lambda) + \xi = J_R(\bar\pi) + Bm + \xi.
$$
On the other hand, the feasibility of $\pi_c$ for $\mathcal{P}(\bm\tau)$ gives
$$
L(\pi_c, \bar{\bm\lambda})
=
J_R(\pi_c) + \bar{\bm\lambda} \cdot (\bm\tau - J_{\bm{C}}(\pi_c))
\geq
J_R(\pi_c).
$$
Combining the previous two inequalities, we get
$$
Bm + \xi
\geq
J_R(\pi_c) - J_R(\bar\pi)
\geq
-\frac{1}{1 - \gamma}
$$
where the last inequality uses $0 \leq J_R(\cdot) \leq \frac{1}{1 - \gamma}$.
Rearranging, and using the fact that $m \leq \tau_i - J_{C_i}(\bar\pi)$ for all $i = 1, \dots, I$, we get
$$
\tau_i - J_{C_i}(\bar\pi) \geq m \geq
-\frac{1}{B(1 - \gamma)} - \frac{\xi}{B}
$$
for all $i = 1, \dots, I$.
and it follows that
$$
J_{C_i}(\bar\pi) \leq \tau_i + \frac{\xi}{B} + \frac{1}{B(1 - \gamma)}.
$$
\end{proof}

Now, we study the case where PDCA is run with a tightened cost threshold $\bm\tau - \eta \bm{1}$ where $\eta \geq 0$.
We denote by $L_\eta(\pi, \bm\lambda) = J_R(\pi) + \bm\lambda \cdot (\bm\tau - \eta \bm{1} - J_{\bm{C}}(\pi))$ the Lagrangian for the tightened problem $\mathcal{P}(\bm\tau - \eta \bm{1})$.
The following lemma shows the property of a $\xi$-near saddle point for $L_\eta(\cdot, \cdot)$.

\begin{lemma} \label{lemma:key2}
Assume that Slater's condition (Assumption~\ref{assumption:slater}) holds and that $\eta < \frac{\phi}{1 - \gamma}$ so that $\mathcal{P}(\bm\tau - \eta \bm{1})$ also satisfies Slater's condition.
Suppose $(\bar\pi, \bar{\bm\lambda})$ is a $\xi$-near saddle point for $L_\eta(\cdot, \cdot)$ with respect to $\widetilde\Pi \times B \bm\Delta^I$.
Let $(\pi_\eta^\star, \bm\lambda_\eta^\star)$ be a primal-dual solution to $\mathcal{P}(\bm\tau - \eta \bm{1})$ and $\pi_\eta^\star \in \widetilde\Pi$.
Assume $B > \Vert \bm\lambda_\eta^\star \Vert_1$.
Then, we have
\begin{align}
J_R(\bar\pi) &\geq J_R(\pi_\eta^\star) - \xi \tag{Optimality} \\
J_{C_i}(\bar\pi) &\leq \tau_i - \eta + \frac{\xi}{B - \Vert \bm\lambda_\eta^\star \Vert_1},\quad \text{for all}~i = 1, \dots, I \tag{Feasibility}
\end{align}
\end{lemma}
\begin{proof}
We first prove near optimality of $\bar\pi$.

\paragraph{Optimality}

Since $(\bar\pi, \bar{\bm\lambda})$ is a $\xi$-near saddle point for $L_\eta(\cdot, \cdot)$ with respect to $\widetilde\Pi \times B \bm\Delta^I$ and $\pi_\eta^\star \in \widetilde\Pi$, we have $L_\eta(\pi_\eta^\star, \bar{\bm\lambda}) \leq L_\eta(\bar\pi, \bm\lambda) + \xi$ for all $\bm\lambda \in B \bm\Delta^I$.
Choosing $\bm\lambda = \bm{0}$, we get
$$
L_\eta(\pi_\eta^\star, \bar{\bm\lambda}) \leq L_\eta(\bar\pi, 0) + \xi = J_R(\bar\pi) + \xi.
$$
Rearranging, we get
$$
J_R(\bar\pi)
\geq
J_R(\pi_\eta^\star) + \bar{\bm\lambda} \cdot (\bm\tau - \eta \bm{1} - J_{\bm{C}}(\pi_\eta^\star)) - \xi
\geq
J_R(\pi_\eta^\star) - \xi
$$
where the second inequality uses the feasibility of $\pi_\eta^\star$ for $\mathcal{P}(\bm\tau - \eta \bm{1})$.
Now, we prove feasibility of $\bar\pi$.

\paragraph{Feasibility}

Recall that $(\pi_\eta^\star, \bm\lambda_\eta^\star)$ is a primal-dual solution to the optimization problem $\mathcal{P}(\bm\tau - \eta \bm{1})$ and $L_\eta(\cdot, \cdot)$ is the Lagrangian function corresponding to the problem $\mathcal{P}(\bm\tau - \eta \bm{1})$.
By strong duality, $(\pi_\eta^\star, \bm\lambda_\eta^\star)$ is a saddle point for $L_\eta(\cdot, \cdot)$ with respect to $\text{Conv}(\Pi) \times \mathbb{R}_+^I$.
Hence, we have
$$
L_\eta(\bar\pi, \bm\lambda_\eta^\star)
\leq
L_\eta(\pi_\eta^\star, \bm\lambda_\eta^\star)
=
J_R(\pi_\eta^\star) + \bm\lambda_\eta^\star \cdot (\bm\tau - \eta \bm{1} - J_{\bm{C}}(\pi_\eta^\star))
=
J_R(\pi_\eta^\star)
$$
where the first inequality follows from the fact that $(\pi_\eta^\star, \bm\lambda_\eta^\star)$ is a saddle point of $L_\eta(\cdot, \cdot)$ and the last equality follows from the complementary slackness property of the solution $(\pi_\eta^\star, \bm\lambda_\eta^\star)$.
Rearranging, we get
\begin{equation} \label{eqn:perf-diff-lower-bound}
J_R(\pi_\eta^\star) - J_R(\bar\pi)
\geq
\bm\lambda_\eta^\star \cdot (\bm\tau - \eta \bm{1} - J_{\bm{C}}(\bar\pi))
\geq
(m - \eta) \Vert \bm\lambda_\eta^\star \Vert_1
\end{equation}
where we define $m = \min_{i \in [I]} (\tau_i - J_{C_i}(\bar\pi))$.
Now, to upper bound $J_R(\pi_\eta^\star) - J_R(\bar\pi)$, we first use the feasibility of $\pi_\eta^\star$ for $\mathcal{P}(\bm\tau - \eta \bm{1})$ as follows.
$$
L_\eta(\pi_\eta^\star, \bar{\bm\lambda})
=
J_R(\pi_\eta^\star) + \bar{\bm\lambda} \cdot (\bm\tau - \eta \bm{1} - J_{\bm{C}}(\pi_\eta^\star))
\geq
J_R(\pi_\eta^\star).
$$
On the other hand, since $(\bar\pi, \bar{\bm\lambda})$ is a $\xi$-near saddle point for $L_\eta(\cdot, \cdot)$ with respect to $\widetilde\Pi \times B \bm\Delta^I$ and $\pi_\eta^\star \in \widetilde\Pi$, we have
$L_\eta(\pi_\eta^\star, \bar{\bm\lambda}) \leq L_\eta(\bar\pi, \bm\lambda) + \xi$ for any $\bm\lambda \in B \bm\Delta^I$.
By choosing $\bm\lambda$ such that $\lambda_j = B$ for $j = \argmin_{i \in [I]} (\tau_i - J_{C_i}(\bar\pi))$ and recalling $m = \min_{i \in [I]} (\tau_i - J_{C_i}(\bar\pi))$, we get
$$
L_\eta(\pi_\eta^\star, \bar{\bm\lambda})
\leq
L_\eta(\bar\pi, \bm\lambda) + \xi
= J_R(\bar\pi) + B (m - \eta) + \xi.
$$
Combining the previous two results (upper bound and lower bound of $L_\eta(\pi_\eta^\star, \bar{\bm\lambda})$), we get
\begin{equation} \label{eqn:perf-diff-upper-bound}
J_R(\pi_\eta^\star) - J_R(\bar\pi) \leq B (m - \eta) + \xi.
\end{equation}
Combining the lower bound (\ref{eqn:perf-diff-lower-bound}) and the upper bound (\ref{eqn:perf-diff-upper-bound}) of $J_R(\pi_\eta^\star) - J_R(\bar\pi)$ and rearranging, we get
$$
m - \eta \geq \frac{- \xi}{B - \Vert \bm\lambda_\eta^\star \Vert_1}.
$$
Since $\tau_i - \eta - J_{C_i}(\bar\pi) \geq m - \eta$ for all $i \in [I]$, rearranging the above gives
$$
J_{C_i}(\bar\pi) \leq \tau_i - \eta + \frac{\xi}{B - \Vert \bm\lambda_\eta^\star \Vert_1}
$$
for all $i = 1, \dots, I$.
\end{proof}
Note that the lemma above requires $B > \Vert \bm\lambda_\eta^\star \Vert_1$.
We will show in Theorem~\ref{thm:main-3}  that with Slater's condition, we can upper bound $\Vert \bm\lambda_\eta^\star \Vert_1$ so that we can choose $B$ that indeed satisfies $B > \Vert \bm\lambda_\eta^\star \Vert_1$.

\section{PROOF OF MAIN RESULTS} \label{appendix:main-results}

\subsection{Proof of Theorem~\ref{thm:main-1}}

We restate the theorem for convenience:

\mainthm*

\begin{proof}
Recall from Lemma~\ref{lemma:concentration} that $\epsilon_{\text{stat}} \coloneqq \frac{C_{\ell_2}}{1 - \gamma} \sqrt{\frac{\log (I \vert \mathcal{F} \vert \vert \Pi \vert \vert \mathcal{W} \vert / \delta)}{n}} + \frac{C_\infty \log(I \vert \mathcal{F} \vert \vert \Pi \vert \vert \mathcal{W} \vert / \delta)}{(1 - \gamma)n}$.
The bound on $n$ guarantees
$$
\epsilon_{\text{stat}} \leq \mathcal{O}((1 - \gamma) \varphi\,\epsilon).
$$
Invoking Lemma~\ref{lemma:saddle-point} with cost threshold $\bm\tau$ and bound $B = 1 + \frac{1}{\varphi}$, we get with probability at least $1 - \delta$ that
$$
L(\pi, \bar{\bm\lambda}) \leq L(\bar\pi, \bm\lambda) + \epsilon_{\text{saddle}}
$$
for all $\pi \in \text{Conv}(\Pi)$ with $w^\pi \in \mathcal{W}$ and $\bm\lambda \in B \bm\Delta^I$ where $\epsilon_{\text{saddle}} \coloneqq \frac{1 + 2B}{(1 - \gamma)^2} \epsilon_{\text{opt}}(K) + \mathcal{O}\left(
\frac{B \epsilon_{\text{stat}}}{1 - \gamma}
\right)$.
Since PDCA chooses $K$ such that $\frac{1 + 2B}{(1 - \gamma)^2} \epsilon_{\text{opt}}(K) \leq \epsilon$ and the bound on $n$ guarantees $\epsilon_{\text{stat}} \leq \mathcal{O}(( 1 - \gamma) \varphi \, \epsilon)$, we have $\epsilon_{\text{saddle}} \leq \mathcal{O}(\epsilon)$.
Hence, invoking Lemma~\ref{lemma:key2} with $\xi = \epsilon_{\text{saddle}}$, $B = 1 + \frac{1}{\varphi}$ and $\eta = 0$, we get
\begin{align*}
J_R(\bar\pi) &\geq J_R(\pi_c) - \epsilon_{\text{saddle}} \geq J_R(\pi_c) - \mathcal{O}(\epsilon) \\
J_{C_i}(\bar\pi) &\leq \tau_i + \frac{\epsilon_{\text{saddle}}}{1 + \frac{1}{\varphi} - \Vert \bm\lambda^\star \Vert_1} \leq \tau_i + \mathcal{O}(\epsilon), \quad i = 1, \dots, I
\end{align*}
where $\bm\lambda^\star$ is the optimal dual variable for the problem $\mathcal{P}(\bm\tau)$ and the last inequality uses $\Vert \bm\lambda^\star \Vert_1 \leq \frac{1}{\varphi}$, which follows by the Slater's condition (Assumption~\ref{assumption:slater}) and Lemma~\ref{lemma:dual-variable-bound}.

\end{proof}

\subsection{Result for arbitrary competing policy}

\begin{theorem} \label{thm:main-2}
Under assumptions~\ref{assumption:value-realizability},\ref{assumption:concentrability},\ref{assumption:w-boundedness} and \ref{assumption:miw-realizability},
the policy $\bar\pi$ returned by the PDCA algorithm (Algorithm~\ref{alg:pdapc}) with the cost threshold $\bm\tau$ and bound $B = \frac{1}{(1 - \gamma)\epsilon}$ satisfies $J_{C_i}(\bar\pi) \leq \tau_i + \mathcal{O}(\epsilon)$ for all $i = 1, \dots, I$, and $J_R(\pi) \geq J_R(\pi^\star) - \mathcal{O}(\epsilon)$ with probability at least $1 - \delta$ where $\pi^\star$ the optimal policy for $\mathcal{P}(\bm\tau)$ as long as
$$
n \geq \Omega\left(\frac{(C_{\ell_2})^2 \log(I \vert \mathcal{F} \vert \vert \Pi \vert \vert \mathcal{W} \vert / \delta)}{(1 - \gamma)^{6} \epsilon^4}\right).
$$
\end{theorem}
\begin{proof}
Recall from Lemma~\ref{lemma:concentration} that $\epsilon_{\text{stat}} \coloneqq \frac{C_{\ell_2}}{1 - \gamma} \sqrt{\frac{\log (I \vert \mathcal{F} \vert \vert \Pi \vert \vert \mathcal{W} \vert / \delta)}{n}} + \frac{C_\infty \log(I \vert \mathcal{F} \vert \vert \Pi \vert \vert \mathcal{W} \vert / \delta)}{(1 - \gamma)n}$.
The choice $n \geq \Omega\left(\frac{(C_{\ell_2})^2 \log(I \vert \mathcal{F} \vert \vert \Pi \vert \vert \mathcal{W} \vert / \delta)}{(1 - \gamma)^{6} \epsilon^4}\right)$ guarantees
$$
\epsilon_{\text{stat}} \leq \mathcal{O}((1 - \gamma)^2 \epsilon^2).
$$
Invoking Lemma~\ref{lemma:saddle-point} with cost threshold $\bm\tau$ and bound $B = \frac{1}{(1 - \gamma) \epsilon}$, we get with probability at least $1 - \delta$ that
$$
L(\pi, \bar{\bm\lambda}) \leq L(\bar\pi, \bm\lambda) + \epsilon_{\text{saddle}}
$$
for all $\pi \in \text{Conv}(\Pi)$ and $\bm\lambda \in \frac{1}{(1 - \gamma)\epsilon} \bm\Delta^I$ where $\epsilon_{\text{saddle}} \coloneqq \frac{1 + 2B}{(1 - \gamma)^2} \epsilon_{\text{opt}}^\pi(K) + \frac{2B}{1 - \gamma} \epsilon_{\text{opt}}^\lambda(K) + \mathcal{O}\left(
\frac{B \epsilon_{\text{stat}}}{1 - \gamma}
\right)$.
Since PDCA chooses $K$ such that $\frac{1 + 2B}{(1 - \gamma)^2} \epsilon_{\text{opt}}^\pi(K) \leq \epsilon$ and $\frac{2B}{1 - \gamma} \epsilon_{\text{opt}}^\lambda(K) \leq \epsilon$, and $n$ is chosen to guarantee $\epsilon_{\text{stat}} \leq \mathcal{O}((1 - \gamma)^2 \epsilon^2)$, we have $\epsilon_{\text{saddle}} \leq \mathcal{O}(\epsilon)$.
Hence, invoking Lemma~\ref{lemma:key1} with $\xi = \epsilon_{\text{saddle}}$ and $B = \frac{1}{(1 - \gamma)\epsilon}$, we get
\begin{align*}
J_R(\bar\pi) &\geq J_R(\pi_c) - \epsilon_{\text{saddle}} \geq J_R(\pi_c) - \mathcal{O}(\epsilon) \\
J_{C_i}(\bar\pi) &\leq \tau_i + \frac{\epsilon_{\text{saddle}}}{B} + \frac{1}{B(1 - \gamma)} \leq \tau_i + \mathcal{O}(\epsilon), \quad i = 1, \dots, I.
\end{align*}

\end{proof}

\subsection{Learning policy satisfying constraints exactly}

Note that results in previous sections provide a bound on sample complexity for finding a nearly optimal policy that \textit{approximately} satisfies the constraints.
In this section, we provide a bound for finding a nearly optimal policy that satisfies the constraints \textit{exactly} by running PDCA with tightened constraints.
We need the following additional technical assumption on MIW realizability.

\begin{assumption} \label{assumption:two-policy-miw}
    Suppose the Slater's condition holds (Assumption~\ref{assumption:slater}). For some constant $\alpha \in [\frac{\varphi}{c(1 - \gamma)}, \frac{\varphi}{1 - \gamma}]$ where $c \geq 0$, we have $w^{\pi^\star_\alpha} \in \mathcal{W}$ where $\pi^\star_\alpha$ denotes an optimal policy of the optimization problem $\mathcal{P}(\bm\tau - \alpha \bm{1})$.
\end{assumption}

\begin{theorem} \label{thm:main-3}
Let $\epsilon \in (0, \frac{1}{2}]$ be given.
Under assumptions~\ref{assumption:value-realizability},\ref{assumption:slater},\ref{assumption:concentrability},\ref{assumption:w-boundedness},\ref{assumption:miw-realizability} and \ref{assumption:two-policy-miw},
the policy $\bar\pi$ returned by the PDCA algorithm (Algorithm~\ref{alg:pdapc}) with the cost threshold $\bm\tau - \eta \bm{1}$, where $\eta = \varphi \epsilon$, and bound $B = \frac{5}{\varphi}$ satisfies $J_{C_i}(\bar\pi) \leq \tau_i$ for all $i = 1, \dots, I$, and $J_R(\pi) \geq J_R(\pi^\star) - \mathcal{O}(\epsilon)$ with probability at least $1 - \delta$, where $\pi^\star$ is an optimal policy for $\mathcal{P}(\bm\tau)$ as long as
$$
n \geq \Omega\left(\frac{(C_{\ell_2})^2 \log(I \vert \mathcal{F} \vert \vert \Pi \vert \vert \mathcal{W} \vert / \delta)}{(1 - \gamma)^{4} \varphi^2 \epsilon^2}\right).
$$
\end{theorem}
\begin{proof}
Recall from Lemma~\ref{lemma:concentration} that $\epsilon_{\text{stat}} \coloneqq \frac{C_{\ell_2}}{1 - \gamma} \sqrt{\frac{\log (I \vert \mathcal{F} \vert \vert \Pi \vert \vert \mathcal{W} \vert / \delta)}{n}} + \frac{C_\infty \log(I \vert \mathcal{F} \vert \vert \Pi \vert \vert \mathcal{W} \vert / \delta)}{(1 - \gamma)n}$.
The bound on $n$ in the theorem guarantees $\epsilon_{\text{stat}} \leq \mathcal{O}((1 - \gamma) \varphi\,\epsilon)$.
Invoking Lemma~\ref{lemma:saddle-point} with cost threshold $\bm\tau - \eta \bm{1}$ and bound $B = \frac{5}{\varphi}$, we get with probability at least $1 - \delta$ that
\begin{equation} \label{eqn:saddle-point-1}
L_\eta(\pi, \bar{\bm\lambda}) \leq L_\eta(\bar\pi, \bm\lambda) + \epsilon_{\text{saddle}}
\end{equation}
for all $\pi \in \text{Conv}(\Pi)$ satisfying $w^\pi \in \mathcal{W}$ and $\bm\lambda \in \frac{5}{\varphi} \bm\Delta^I$ where $L_\eta(\cdot, \cdot)$ is the Lagrangian for $\mathcal{P}(\bm\tau - \eta \bm{1})$ and $\epsilon_{\text{saddle}} \coloneqq \frac{1 + 2B}{(1 - \gamma)^2} \epsilon_{\text{opt}}(K) + \mathcal{O}\left(\frac{\epsilon_{\text{stat}}}{(1 - \gamma)\varphi}\right)$.
Since PDCA chooses $K$ such that $\frac{1 + 2B}{(1 - \gamma)^2} \epsilon_{\text{opt}}(K) \leq \epsilon$ and $n$ is chosen to guarantee $\epsilon_{\text{stat}} \leq \mathcal{O}((1 - \gamma) \varphi \, \epsilon)$, we have $\epsilon_{\text{saddle}} \leq 2 \epsilon$ (with appropriate scaling of $n$ by a universal constant).

Let $\pi^\star$ be an optimal policy for $\mathcal{P}(\bm\tau)$ and $\pi^\star_\alpha$ for $\mathcal{P}(\bm\tau - \alpha \bm{1})$.
By the MIW realizability assumptions \ref{assumption:miw-realizability} and \ref{assumption:two-policy-miw}, we have $w^{\pi^\star}, w^{\pi^\star_\alpha} \in \mathcal{W}$ and it follows from (\ref{eqn:saddle-point-1}) that
\begin{align}
L_\eta(\pi^\star, \bar{\bm\lambda}) &\leq L_\eta(\bar\pi, \bm\lambda) + 2 \epsilon \label{eqn:two-policy-saddle1} \\
L_\eta(\pi^\star_\alpha, \bar{\bm\lambda}) &\leq L_\eta(\bar\pi, \bm\lambda) + 2 \epsilon \label{eqn:two-policy-saddle2}
\end{align}
for all $\bm\lambda \in \frac{5}{\varphi} \bm\Delta^I$.
Now, we show near optimality and exact feasibility from these inequalities.

\paragraph{Near Optimality}

Setting $\bm\lambda = \bm{0}$ in (\ref{eqn:two-policy-saddle1}) and rearranging, we get
\begin{align*}
J_R(\bar\pi)
&\geq
J_R(\pi^\star) + \bar{\bm\lambda} \cdot (\bm\tau - \eta \bm{1} - J_{\bm{C}}(\pi^\star)) - 2\epsilon \\
&\geq
J_R(\pi^\star) - \eta \Vert \bar{\bm\lambda} \Vert_1 - \mathcal{O}(\epsilon) \\
&\geq
J_R(\pi^\star) - \mathcal{O}(\epsilon)
\end{align*}
where the second inequality follows by the feasibility of $\pi^\star$ for $\mathcal{P}(\bm\tau)$ and $\epsilon_{\text{saddle}} \leq \mathcal{O}(\epsilon)$;
the last inequality follows by $\eta \Vert \bar{\bm\lambda} \Vert_1 \leq \eta B = \mathcal{O}(\epsilon)$.
This proves near optimality of $\bar\pi$
Now we prove that $\bar\pi$ is (exactly) feasible for $\mathcal{P}(\bm\tau)$.

\paragraph{Exact Feasibility}
Define $m = \min_{i \in [I]} (\tau_i - J_{C_i}(\bar\pi))$.
If $m \geq 0$ then $\tau_i - J_{C_i}(\bar\pi) \geq 0$ for all $i = 1, \dots, I$ and exact feasibility trivially holds.
We only consider the case where $m < 0$.
Define a mixture policy $\widetilde\pi = (1 - \zeta) \pi^\star + \zeta \pi_\alpha^\star$ where $\zeta \in (0, 1)$ is to be determined later.
Since $L_\eta(\cdot, \bm{\bar\lambda})$ is linear, a linear combination of (\ref{eqn:two-policy-saddle1}) and (\ref{eqn:two-policy-saddle2}) with coefficients $1 - \zeta$ and $\zeta$ respectively, we get
$$
L_\eta(\widetilde\pi, \bar{\bm\lambda}) \leq L_\eta(\bar\pi, \bm\lambda) + 2\epsilon.
$$
Choosing $\bm\lambda$ such that $\lambda_j = B$ for $j = \argmin_{i \in [I]} (\tau_i - J_{C_i}(\bar\pi))$ and $\lambda_j = 0$ for all other indices, we get
\begin{align*}
L_\eta(\widetilde\pi, \bar{\bm\lambda})
&\leq
J_R(\bar\pi) + \bm\lambda \cdot (\bm\tau - \eta \bm{1} - J_{\bm{C}}(\bar\pi)) + 2\epsilon \\
&=
J_R(\bar\pi) + B(m - \eta) + 2\epsilon.
\end{align*}
On the other hand, using the fact that $\widetilde\pi$ is feasible for $\mathcal{P}(\bm\tau - \zeta \alpha \bm{1})$, we get
\begin{align*}
L_\eta(\widetilde\pi, \bar{\bm\lambda})
&=
J_R(\widetilde\pi) + \bar{\bm\lambda}(\bm\tau - \eta \bm{1} - J_{\bm{C}}(\widetilde\pi)) \\
&\geq
J_R(\widetilde\pi) + (\zeta \alpha - \eta) \Vert \bar{\bm\lambda} \Vert_1.
\end{align*}
Combining the previous two results (upper bound and lower bound of $L_\eta(\widetilde\pi, \bar{\bm\lambda})$) and rearranging, we get
\begin{equation}
J_R(\widetilde\pi) - J_R(\bar\pi) \leq B(m - \eta) - (\zeta \alpha - \eta) \Vert \bar{\bm\lambda} \Vert_1 + 2\epsilon. \label{eqn:exact-feasibility-intermediate}
\end{equation}
Now, to get a lower bound of $J_R(\widetilde\pi) - J_R(\bar\pi)$, let $(\widetilde\pi^\star, \widetilde{\bm\lambda}^\star)$ be a primal-dual solution of $\mathcal{P}(\bm\tau - \zeta \alpha \bm{1})$.
Note that $\mathcal{P}(\bm\tau - \zeta \alpha \bm{1})$ is feasible by the Slater's condition assumption \ref{assumption:slater} and the fact that $\zeta \alpha \in (0, \frac{\varphi}{1 - \gamma})$.
Since $(\widetilde\pi^\star, \widetilde{\bm\lambda}^\star)$ is a saddle point of $L_{\zeta \alpha}(\pi, \bm\lambda)
= J_R(\pi) + \bm\lambda \cdot (\bm\tau - \zeta \alpha \bm{1} - J_{\bm{C}}(\pi))$ with respect to $\text{Conv}(\Pi) \times \mathbb{R}_+^I$, we get
$$
L_{\zeta \alpha}(\bar\pi, \widetilde{\bm\lambda}^\star) \leq L_{\zeta \alpha}(\widetilde\pi^\star, \widetilde{\bm\lambda}^\star) = J_R(\widetilde\pi^\star) \leq J_R(\pi^\star) \leq \frac{1}{1 - \zeta} J_R(\widetilde\pi)
$$
where the equality follows by the complementary slackness property;
the second inequality follows since the feasibility set of $\mathcal{P}(\bm\tau)$ contains that of $\mathcal{P}(\bm\tau - \zeta \alpha \bm{1})$;
and the last inequality follows by $J_R(\widetilde\pi) = (1 - \zeta) J_R(\pi^\star) + \zeta J_R(\pi_\alpha^\star) \geq (1 - \zeta) J_R(\pi^\star)$.
Rearranging, we get
\begin{align*}
J_R(\widetilde\pi) - J_R(\bar\pi)
&\geq - \zeta J_R(\bar\pi) + (1 - \zeta) \widetilde{\bm\lambda}^\star \cdot (\bm\tau - \zeta \alpha \bm{1} - J_{\bm{C}}(\bar\pi)) \\
&\geq
\frac{-\zeta}{1 - \gamma} + (1 - \zeta) (m - \zeta \alpha) \Vert \widetilde{\bm\lambda}^\star \Vert_1
\end{align*}
where the second inequality follows by $J_R(\cdot) \leq \frac{1}{1 - \gamma}$ and the definition of $m$.
Combining with the upper bound of $J_R(\widetilde\pi) - J_R(\bar\pi)$ shown in (\ref{eqn:exact-feasibility-intermediate}) and rearranging, we get
\begin{equation} \label{eqn:exact-feasibility-intermediate2}
(B - (1 - \zeta) \Vert \widetilde{\bm\lambda}^\star \Vert_1)m
\geq
B\eta + (\zeta \alpha - \eta) \Vert \bar{\bm\lambda} \Vert_1 - \frac{\zeta}{1 - \gamma} - (1 - \zeta) \zeta \alpha \Vert \widetilde{\bm\lambda}^\star \Vert_1 - 2\epsilon.
\end{equation}
Now, we choose our parameters as follows.
$$
\zeta = (1 - \gamma)\epsilon, \quad
B = \frac{5c}{\varphi}, \quad
\eta = \frac{\varphi \epsilon}{c}.
$$
Note that $B\eta = 5\epsilon$ and $\zeta \alpha \geq \eta$.
Also, since $\widetilde{\bm\lambda}^\star$ is a dual solution of $\mathcal{P}(\bm\tau - \zeta \alpha \bm{1})$, which has a margin of $\frac{\varphi}{1 - \gamma} - \zeta \alpha$, Lemma~\ref{lemma:dual-variable-bound} gives $\Vert \widetilde{\bm\lambda}^\star \Vert_1 \leq \frac{1}{\varphi - (1 - \gamma) \zeta \alpha}$.
Hence,
$$
\zeta \alpha \Vert \widetilde{\bm\lambda}^\star \Vert_1
\leq
\frac{\zeta \alpha}{\varphi - (1 - \gamma) \zeta \alpha}
\leq
\frac{1}{1 - \gamma} \frac{\zeta \varphi}{\varphi - \varphi \zeta}
\leq \frac{2\zeta}{1 - \gamma}
\leq 2\epsilon
$$
where the second inequality uses $\alpha \leq \frac{\varphi}{1 - \gamma}$ and the fact that $h(x) = \frac{x}{\varphi - x}$ is increasing for $x \in (0, \varphi)$;
and the third inequality uses the fact that $\zeta \leq \frac{1}{2}$.
Note that $\Vert \widetilde{\bm\lambda}^\star \Vert_1 \leq \frac{\epsilon}{\zeta \alpha} = \frac{1}{(1 - \gamma)\alpha} < B$ so that $B - (1 - \zeta) \Vert \widetilde{\bm\lambda}^\star \Vert_1 > 0$.
Hence, the previous result (\ref{eqn:exact-feasibility-intermediate2}) gives
\begin{align*}
(B - (1 - \zeta) \Vert \widetilde{\bm\lambda}^\star \Vert_1)m
&\geq
B\eta + (\zeta \alpha - \eta) \Vert \bar{\bm\lambda} \Vert_1 - \frac{\zeta}{1 - \gamma} - (1 - \zeta) \zeta \alpha \Vert \widetilde{\bm\lambda}^\star \Vert_1 - 2\epsilon \\
&\geq
5\epsilon + 0 - \epsilon - 2\epsilon - 2\epsilon \\
&=
0.
\end{align*}
Since $B - (1 - \zeta) \Vert \widetilde{\bm\lambda}^\star \Vert_1 > 0$, we have $m \geq 0$ which implies $\tau_i - J_{C_i}(\bar\pi) \geq 0$ for all $i = 1, \dots, I$.
This completes the proof.

\end{proof}

\section{CONVEX OPTIMIZATION}

\begin{lemma} \label{lemma:perturbation-analysis}
Let $(\pi^\star, \lambda^\star)$ be optimal primal dual solutions to the constrained optimization problem $\mathcal{P}(\tau)$.
Let $(\widetilde\pi^\star, \widetilde\lambda^\star)$ be optimal primal dual solutions to the perturbed problem $\mathcal{P}(\widetilde{\bm\tau})$ where $\widetilde{\bm\tau} = \bm\tau - \eta \bm{1}$.
Then, we have
$$
J_R(\widetilde\pi^\star) \geq J_R(\pi^\star) - \eta \Vert \widetilde{\bm\lambda}^\star \Vert_1.
$$
\end{lemma}
\begin{proof}
The proof follows Chapter 5.6 in \textcite{boyd_convex_2004}.
By strong duality of the optimization problem $\mathcal{P}(\widetilde{\bm\tau})$, we have $J_R(\widetilde\pi^\star) = d(\widetilde{\bm\lambda}^\star)$ where $d(\bm\lambda) = \max_{\pi \in \text{Conv}(\Pi)} L(\pi, \bm\lambda; \widetilde{\bm\tau})$ is the dual function of $\mathcal{P}(\widetilde{\bm\tau})$.
Hence,
\begin{align*}
J_R(\widetilde\pi^\star)
&=
d(\widetilde{\bm\lambda}^\star) \\
&\geq
J_R(\pi^\star) + \widetilde{\bm\lambda}^\star \cdot (\widetilde{\bm\tau} - J_{\bm{C}}(\pi^\star)) \\
&=
J_R(\pi^\star) + \widetilde{\bm\lambda}^\star \cdot (\bm\tau - J_{\bm{C}}(\pi^\star)) - \widetilde{\bm\lambda}^\star \cdot \eta \bm{1} \\
&\geq
J_R(\pi^\star) - \widetilde{\bm\lambda}^\star \cdot \eta \bm{1}
\end{align*}
where the first inequality follows from the definition of the dual function $d(\cdot)$ and the second follows by the feasibility of $\pi^\star$ for $\mathcal{P}(\bm\tau)$.
\end{proof}

\begin{lemma} \label{lemma:dual-variable-bound}
Consider a constrained optimization problem $\mathcal{P}(\bm\tau)$ with threshold $\bm\tau = (\tau_1, \dots, \tau_I)$ with $\tau_i > 0$ for all $i = 1, \dots, I$.
Suppose the problem satisfies Slater's condition with margin $\frac{\varphi}{1 - \gamma} > 0$, in other words, there exists $\pi \in \Pi$ that satisfies the constraint $J_{C_i}(\pi) \leq \tau_i - \frac{\varphi}{1 - \gamma}$ for all $i = 1, \dots, I$.
Then, the optimal dual variable $\bm\lambda^\star$ of the problem satisfies $\Vert \bm\lambda^\star \Vert_1 \leq \frac{1}{\varphi}$.
\end{lemma}
\begin{proof}
Let $\pi^\star$ be an optimal policy of the optimization problem $\mathcal{P}(\bm\tau)$.
Define the dual function $f(\bm\lambda) = \max_{\pi \in \Pi} J_R(\pi) + \bm\lambda \cdot (\bm\tau - J_{\bm{C}}(\pi))$.
Let $\bm\lambda^\star = \argmin_{\bm\lambda \in \mathbb{R}_+^I} f(\bm\lambda)$.
Trivially, $\lambda_i^\star \geq 0$ for all $i = 1, \dots, I$.
Also, by strong duality, we have $f(\bm\lambda^\star) = J_R(\pi^\star)$.
Let $\widehat{\pi}$ be a feasible policy with $J_{\bm{C}}(\widehat{\pi}) \leq \bm\tau - \frac{\varphi}{1 - \gamma} \bm{1}$ where the inequality is component-wise and $\bm{1} = (1, \dots, 1)$.
Such a policy exists by the assumption of this lemma.
Then,
$$
J_R(\pi^\star)
= f(\bm\lambda^\star)
\geq J_R(\widehat\pi) + \bm\lambda^\star \cdot (\bm\tau - J_{\bm{C}}(\widehat\pi))
\geq J_R(\widehat\pi) + \bm\lambda^\star \cdot \frac{\varphi}{1 - \gamma} \bm{1} = J_R(\widehat\pi) + \frac{\varphi}{1 - \gamma} \Vert \bm\lambda^\star \Vert_1.
$$
Rearranging and using $1 / (1 - \gamma) \geq J_R(\pi^\star) \geq J_R(\widehat\pi) \geq 0$ completes the proof:
$$
\Vert \bm\lambda^\star \Vert_1 \leq \frac{J_R(\pi^\star) - J_R(\widehat\pi)}{\varphi / (1 - \gamma)} \leq \frac{1}{\varphi}.
$$
\end{proof}

\section{EXPERIMENTS} \label{appendix:experiments}

In this section, we empirically demonstrate the performance of our algorithm PDCA by running it in various environments and comparing the performance to COptiDICE.
For the parameter tuning and the experiments, we used an internal cluster of nodes with 20-core 2.40 GHz CPU and Nvidia Tesla V100 GPU.
The total amount of computing time was around 600 hours.

\subsection{Tabular CMDP Experiments} \label{section:tabular-generation-detail}

In this section, we provide details of the experiment run on a randomly generated CMDP discussed in Section~\ref{section:experiments}.
We follow a similar experimental protocol as \textcite{lee_coptidice_2022}.

\paragraph{CMDP Generation}

We set the number of states to 10 and the number of actions to 5.
The transition probability is randomly generated by drawing from a Dirichlet distribution with all parameters $(1, \dots, 1)$ for generating each $P(\cdot | s, a)$.
We set the number of cost functions $I$ to 1.
The reward function $R$ is randomly drawn from $[0, 1]$ uniformly for each $R(s, a)$.
The cost function $C_1$ is randomly drawn from a beta distribution with parameters 0.2, 0.2 for each $C_1(s, a)$.
We choose the discount factor $\gamma = 0.8$ and the cost threshold $\tau = 0.5$.
We repeat the random generation policy until the cost threshold is not slack for the optimal policy.

\paragraph{$\pi$-Player}

We use the natural policy gradient algorithm with exponential weight updating scheme for the $\pi$ player (Algorithm~\ref{alg:natural-policy}).

\begin{algorithm}
\KwInput{Learning rate $\eta$, sequence of functions $h_1, \dots, h_K \in (\mathcal{S} \times \mathcal{A} \rightarrow [-1, 1])$}
\KwInit{$\pi_1$ a uniform policy.}
\For{$k = 1, \dots, K - 1$}{
  \For{$s \in \mathcal{S}$}{
    $\pi_{k + 1}(a | s) \propto \pi_k(a | s) \exp(\eta h_k(s, a))$ normalized to sum to 1 across $a \in \mathcal{A}$.
  }
}
\KwReturn{Sequence $\pi_1, \dots, \pi_K$}
\caption{Natural policy gradient}
\label{alg:natural-policy}
\end{algorithm}

\paragraph{Offline Dataset}
We set behavior policy $\mu$ to be a mixture
 of $\pi_{\text{uniform}}$ and $\pi^\star$ where $\pi_{\text{uniform}}$ is the uniform policy that takes actions uniformly at random from the action space $\mathcal{A}$ at every state, and $\pi^\star$ is the optimal solution to the generated CMDP.
We exactly solve for the occupancy measure $d^\mu$ of the behavior policy $\mu$.
We repeatedly sample the $(s, a)$ pair from $d^\mu$ and then sample $s'$ according to $P(\cdot | s, a)$.

\paragraph{Hyperparameters}

The learning rate for the $\pi$-player is chosen using grid search in $\{1, 2, 5, 10\}$.
The bound $B$ for the $\lambda$-player is chosen by grid search in $\{2, 5, 10\}$.
The bound $C_\infty$ for $\mathcal{W}$ is chosen by a grid search in $\{2, 5, 10\}$.

\subsection{RWRL Benchmark Experiments} \label{appendix:rwrl}

\paragraph{Hyperparameters}

We do a grid search on $\{ 0.00005, 0.0001, 0.0003, 0.0005, 0.001 \}$ for determining the learning rate $\eta_{\text{fast}}$ for the critics and a grid search on $\{ 0.00005, 0.0001, 0.0002 \}$ for determining the learning rate $\eta_{\text{slow}}$ for the $\pi$-player.
The chosen learning rates are $\eta_{\text{fast}} = 0.0003$ and $\eta_{\text{slow}} = 0.0001$.
We use the batch size 1024.
We run $K = 30000$ iterations for Cartpole, Walker, Quadruped environments and $K=50000$ iterations for the Humanoid environment.
For the policy network and the networks for the critics, we use fully-connected neural networks with two hidden layers of width 256.

\subsection{Bullet Safety Gym Benchmark Experiments} \label{appendix:bullet}

In this section, we provide details of the experiments run on Bullet Safety Gym benchmark environments.

\paragraph{Offline Datasets}

We use the offline datasets provided by \textcite{liu2023datasets}.
They collect dataset for each environment by merging trajectories generated by algorithms trained with various cost thresholds and hyperparameters.
After merging, they run a post-processing of filtering redundant trajectories to ensure a diverse set of trajectories.
For details, refer to their paper.

\paragraph{Hyperparameters}

Following the setup used by \textcite{liu2023datasets}, we set the learning rate $\eta_{\text{fast}}$ for the critics to $0.001$ and the learning rate $\eta_{\text{slow}}$ for the $\pi$-player to $0.0001$.
We use the batch size 512.
We run $K = 100,000$ iterations.
For the policy network and the networks for the critics, we use fully-connected neural networks with two hidden layers of width 256.
We do a grid search on $\{2, 5, 10\}$ for the bound $B$ for the $\lambda$-player.
We do a grid search on $\{2, 5, 10\}$ for the bound $C_\infty$ for $\mathcal{W}$.

\end{document}